\documentclass{article} 

\usepackage[table]{xcolor}

\usepackage[utf8]{inputenc} 
\usepackage[T1]{fontenc}    
\usepackage{hyperref}       
\usepackage{url}            
\usepackage{booktabs}       
\usepackage{amsfonts}       
\usepackage{nicefrac}       
\usepackage{microtype}      
\usepackage{xcolor}         

\usepackage{graphicx} 
\usepackage{subfigure}
\usepackage{bm}

\usepackage{mathtools}
\usepackage{amsmath,amssymb,amsthm,epsfig,epstopdf,titling,url,array}
\theoremstyle{plain}
\usepackage {threeparttable}
\newtheorem{theorem}{Theorem}
\newtheorem{lemma}{Lemma}
\newtheorem{proposition}{Proposition}

\theoremstyle{definition}

\newtheorem{assumption}{Assumption}
\theoremstyle{remark}

\usepackage{cleveref}
\crefname{theorem}{theorem}{theorems}
\crefname{lemma}{lemma}{lemmas}
\crefname{definition}{definition}{definitions}
\crefname{assumption}{assumption}{assumptions}
\crefname{corollary}{cororally}{corollaries}
\crefname{property}{property}{properties} 
\usepackage{makecell, caption, booktabs, multirow, siunitx}
\usepackage{here}
\usepackage{appendix}
\usepackage{etoolbox}
\usepackage {sepnum}
\BeforeBeginEnvironment{appendices}{\clearpage}
\usepackage[noend]{algpseudocode}
\usepackage{wrapfig}
\usepackage{comment}
\usepackage{algorithm}
\usepackage{bookmark}

 \usepackage[main, final]{neurips_2026}

\title{Learning Survival Models with Right-Censored Reporting Delays} 
\author{%
  Yuta Shikuri \\ 
  The Graduate University for Advanced Studies \\  
  Tokio Marine Holdings, Inc. \\ 
  \texttt{shikuriyuta@gmail.com} \\ 
  \And 
  Hironori Fujisawa \\ 
  The Institute of Statistical Mathematics \\ 
  The Graduate University for Advanced Studies \\ 
  RIKEN \\ 
  \texttt{fujisawa@ism.ac.jp} \\ 
}

\begin{document}
\maketitle

\begin{abstract} 
Survival analysis provides statistical methods to model the time until an event occurs. 
Reporting delays arise when event times are not observed at their occurrence but are only revealed upon reporting. 
This issue is particularly critical for timely risk evaluation when the observation window is short due to administrative censoring. 
In this study, we incorporate right-censored reporting delays by jointly modeling parametric hazards for the event and reporting processes.   
We then construct a consistent estimator for the model parameters and develop a Monte Carlo expectation-maximization algorithm to compute it. 
To address the challenges posed by administrative censoring, we leverage these findings and propose a transfer-learning procedure. 
Experimental results demonstrate that our method improves the accuracy of timely risk evaluation under administrative censoring.  
\end{abstract}

\section{Introduction} 
\label{intro}  
Survival analysis is a statistical framework for modeling and estimating the time until an event occurs \citep{Kalbfleisch}. 
Its applications span a wide range of fields, including health sciences, marketing, reliability engineering, and finance. 
Beyond the statistics community, machine-learning methods have been applied to survival analysis \citep{Ping}. 
The proportional hazards model decomposes the hazard function into a time-dependent baseline hazard and a time-independent component \citep{Cox0, Cox1, Breslow0}. 
Modern approaches in the machine-learning community have extended this model with neural networks to improve predictive performance \citep{Scheel, Zhong, Nagpal}. 
\Cref{Cox} reviews the proportional hazards framework and related developments. 

Censoring is a key concept in survival analysis; for some individuals, the exact event time is unknown and only partial information is available. 
For timely risk evaluation, follow-up is often right-censored at a fixed assessment date, a phenomenon known as administrative censoring. 
When the resulting enrollment-to-evaluation window is short, the identifiability condition required for the asymptotic consistency of the parameter estimator may be violated. 
In such settings, borrowing information from related cohorts can supplement the limited information available in the target domain. 
Motivated by this idea, recent studies have explored transfer learning and domain adaptation for survival models \citep{Chandan, Bellot, Carolin, Ethan, Julie}. 

Reporting delays are also important in survival analysis; a key difference from the standard setting is that the event occurrence time is revealed only after reporting. 
The biostatistical literature emphasizes accounting for these reporting delays, since ignoring them can lead to underestimation of the hazard of event occurrence. 
In typical clinical trials, vital status is often ascertained only at scheduled follow-up visits, resulting in a time lag between death occurrence and its reporting. 
To account for such delays, several studies have proposed estimators as alternatives to the naive Kaplan-Meier estimator \citep{Anastasios, Hubbard}. 
In infectious disease research, individuals may remain unobserved until infection is detected and reported, so the at-risk cohort may not be known in advance. 
A series of studies addresses this issue under right truncation \citep{Lagakos, Lawless0, Harris, Lawless, Marcello}. 
Similar delay structures have also been studied beyond survival analysis; see \Cref{delay}. 

\begin{wrapfigure}{r}{0.5\textwidth}
\vspace{-1em}
\centering
\includegraphics[width=0.47\textwidth]{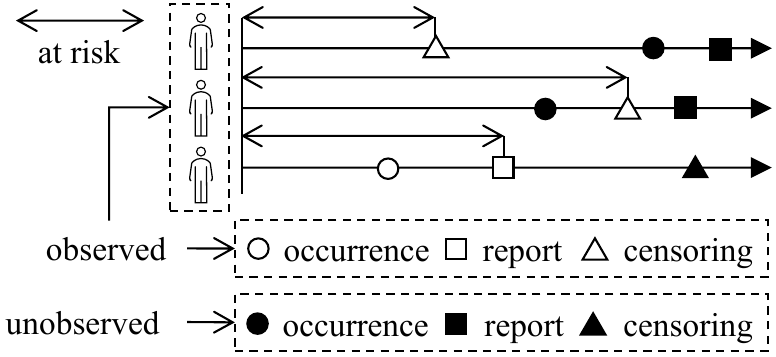}
\caption{
Illustration of right-censored reporting delays. 
The figure shows three possible observation patterns over the study period. 
The detailed setting is provided in \Cref{Setting}. 
}
\label{overview}
\end{wrapfigure}

As illustrated in \Cref{overview}, we consider survival analysis with right-censored reporting delays: individuals are known to be at risk, but the event occurrence time is observed only when it is reported at unscheduled times. 
This observation scheme is plausible in certain real-world applications, where an event of interest triggers an unscheduled report that provides an opportunity to collect detailed information about its occurrence. 
For instance, individuals may be followed from enrollment until death, while disease onset is ascertained only at death based on the recorded cause.  
This example is closely related to survival-sacrifice experiments \citep{Kodell,Bruce,Laan}, with one key distinction: we assume that the exact onset time becomes observable at death. 
Although \citet{Buchardt} considered the same observation scheme, they estimated the event-occurrence and reporting processes separately without examining whether the parameters governing these processes are jointly identifiable. 

\begin{wrapfigure}{r}{0.5\textwidth}
\vspace{-1em}
\centering
\includegraphics[width=0.47\textwidth]{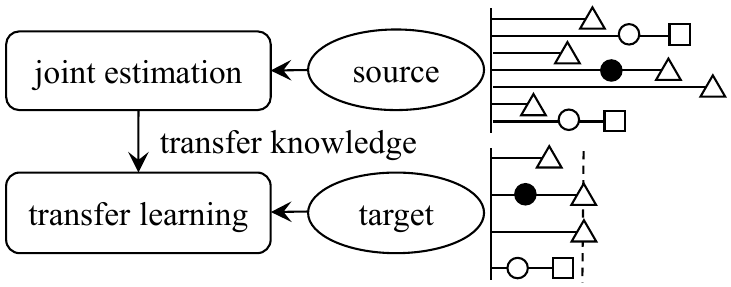}
\caption{ 
Illustration of the transfer-learning procedure. 
The procedure enables consistent estimation of the survival-model parameters in the target domain under administrative censoring. 
The formal setting is given in \Cref{transfer_basic}. 
} 
\label{flow} 
\end{wrapfigure} 

In this study, we develop a survival analysis framework for jointly estimating the hazards of event occurrence and reporting under right-censored reporting delays. 
In the absence of administrative censoring, we establish the asymptotic consistency of the resulting parameter estimator under a mild identifiability condition. 
However, when the enrollment-to-evaluation window is short, administrative censoring makes identification more challenging. 
To address this issue, we adopt the transfer-learning procedure illustrated in \Cref{flow}. 
Specifically, we consider two domains: a source domain without administrative censoring and a target domain subject to it. 
In the first stage, we estimate all parameters using observations from the source domain. 
In the second stage, we transfer selected components of the source-domain parameter estimates to the target domain and estimate the remaining components using target-domain observations. 
Under appropriate transfer conditions, the resulting estimator is consistent. 
As a specialization of this transfer-learning procedure, we derive an approximate estimating equation for parametric proportional hazards models \citep{Holford,Patrick}. 
Experiments demonstrate that our approach improves the accuracy of timely risk evaluation under administrative censoring. 
\Cref{proofs} provides the proofs of the theoretical results.

\section{Standard Survival Analysis} 
\label{Survival Analysis} 
This section reviews standard survival analysis in the absence of reporting delays. 
The derivations of the estimators and the proofs of their consistency are provided in \Cref{sup_preliminaries}, which serves as background for the proofs of our main results. 

Throughout this paper, random variables are denoted by uppercase letters and their realizations by lowercase letters, with indices and accents omitted when no confusion arises. 
Let $P_X$ denote the distribution of the covariates $X \in \mathcal{X} \subset \mathbb{R}^d$ in the domain under consideration. 
All probability density functions are assumed to be strictly positive and bounded. 
The true parameter values are assumed to lie in the interiors of their respective parameter spaces. 
We also assume the standard regularity conditions required for the asymptotic arguments below, including uniform convergence and appropriate conditions for the consistency of the maximizers. 
An asterisk superscript denotes a true parameter value, e.g., $\bm{\theta}^*$ for $\bm{\theta}$, whereas a prime superscript denotes another element of the same parameter space, e.g., $\bm{\theta}^\prime$.  
A semicolon explicitly specifies parameters, as in $f(y; \bm{\theta})$, while a vertical bar indicates conditioning on observed covariates, as in $f(y \mid \bm{x})$. 

Survival analysis models the time until an event occurs as a non-negative random variable $T$. 
For $t \in [0, \infty)$ and $\bm{x} \in \mathcal{X}$, the hazard function of $T$ is defined as $h(t \mid \bm{x}) \coloneqq f(t \mid \bm{x}) S(t \mid \bm{x})^{-1}$, 
where $f(t \mid \bm{x})$ denotes the conditional probability density function of $T$ given $\bm{x}$, and $S(t \mid \bm{x}) \coloneqq 1 - \int_0^t f(s \mid \bm{x}) ds$ is the corresponding survival function. 
The censoring time is modeled as a non-negative random variable $C$, assumed to be conditionally independent of $T$ given $X$. 
Define the observation time as $Y \coloneqq \min \{T, C\}$ and the event indicator as $V \coloneqq \mathbb{I}(T \leq C)$, where $\mathbb{I}$ denotes the indicator function. 
For $y \in [0, \infty)$, $v \in \{0, 1\}$, and $\bm{x} \in \mathcal{X}$, the conditional joint distribution of $(Y, V)$ is given by  
\begin{align} 
\label{mix_G} 
g(y, v \mid \bm{x}) \coloneqq 
\begin{cases} 
\frac{\partial}{\partial y} \Pr(Y \leq y, V = 0 \mid X = \bm{x}) = S(y \mid \bm{x}) f_c(y \mid \bm{x}) &  (v = 0), \\ 
\frac{\partial}{\partial y} \Pr(Y \leq y, V = 1 \mid X = \bm{x}) = f(y \mid \bm{x}) S_c(y \mid \bm{x}) & (v = 1), 
\end{cases} 
\end{align} 
where $f_c$ and $S_c$ denote the conditional probability density function and the survival function of $C$, respectively. 
Administrative censoring refers to situations where observations are censored at a predetermined time point $\tau \in [0, \infty)$. 
Let $\bar{Y} \coloneqq \min\{Y, \tau\}$ and $\bar{V} \coloneqq \mathbb{I}(T \leq C, T \leq \tau)$. 
For $y \in [0, \tau]$, $v \in \{0, 1\}$, and $\bm{x} \in \mathcal{X}$, the conditional joint distribution of $(\bar{Y}, \bar{V})$ is given by  
\begin{align} 
\label{mix_Gbar} 
\bar{g}(y, v \mid \bm{x}) \coloneqq 
\begin{cases} 
\Pr(\bar{Y} = \tau, \bar{V} = 0 \mid X = \bm{x}) = S(\tau \mid \bm{x}) S_c(\tau \mid \bm{x}) &  (y = \tau, v = 0), \\ 
\frac{\partial}{\partial y} \Pr(\bar{Y} \leq y, \bar{V} = 0 \mid X = \bm{x}) = S(y \mid \bm{x}) f_c(y \mid \bm{x}) &  (y \neq \tau, v = 0), \\ 
\frac{\partial}{\partial y} \Pr(\bar{Y} \leq y, \bar{V} = 1 \mid X = \bm{x}) = f(y \mid \bm{x}) S_c(y \mid \bm{x}) & (v = 1).  
\end{cases} 
\end{align} 

In this study, we parameterize the hazard functions of $T$ and $C$ by unknown real-valued vectors $\bm{\theta}$ and $\bm{\theta}_c$, respectively. 
Parametric proportional hazards models are widely used across many fields. 
They specify the hazard function as $h_b(t; \bm{\alpha}) \phi(\bm{x}; \bm{\beta})$ for $t \in [0, \infty)$ and $\bm{x} \in \mathcal{X}$, 
where the baseline hazard $h_b$ and the time-independent component $\phi$ are strictly positive functions parameterized by unknown real-valued vectors $\bm{\alpha}$ and $\bm{\beta}$, respectively. 

The parameters characterizing the hazard functions of $T$ and $C$ are estimated from independent and identically distributed observations $\mathcal{D} \coloneqq (\bm{x}_i, y_i, v_i)_{i=1}^n$. 
Each observation is characterized by the covariates $\bm{x}_i \in \mathcal{X}$, the time $y_i \in [0, \infty)$ until the event or censoring, and the event indicator $v_i \in \{0, 1\}$. 
Then the log likelihood in the absence of administrative censoring is given by $\mathcal{L}(\bm{\theta}) + \mathcal{L}_c(\bm{\theta}_c)$, where the two components are defined as 
\begin{align} 
\label{Likelihood} 
\mathcal{L}(\bm{\theta}) \coloneqq \sum_{i = 1}^n \Bigl((1 - v_i) \log S(y_i \mid \bm{x}_i) + v_i \log f(y_i \mid \bm{x}_i)\Bigr), \\ 
\mathcal{L}_c(\bm{\theta}_c) \coloneqq \sum_{i = 1}^n \Bigl((1 - v_i) \log f_c(y_i \mid \bm{x}_i) + v_i \log S_c(y_i \mid \bm{x}_i)\Bigr).  
\end{align} 
Based on the discrepancy between \Cref{mix_G,mix_Gbar}, we derive the corresponding likelihood function under administrative censoring. 
For simplicity, we do not distinguish between realizations of $(Y,V)$ and $(\bar{Y},\bar{V})$, and collectively denote them by $(y_i,v_i)$.  
In the presence of administrative censoring, $y_i$ is understood as the censored observation taking values in $[0, \tau]$. 
As an alternative to $\mathcal{L}_c$, we introduce   
\begin{align} 
\mathcal{L}_c^\dagger(\bm{\theta}_c) \coloneqq \sum_{i = 1}^n \Bigl((1 - v_i^\dagger) \log f_c(y_i \mid \bm{x}_i) + v_i^\dagger \log S_c(y_i \mid \bm{x}_i)\Bigr),  
\end{align} 
where $v_i^\dagger \coloneqq \mathbb{I}(y_i \geq \tau (1 - v_i))$. 
The log likelihood in the presence of administrative censoring is given by $\mathcal{L}(\bm{\theta}) + \mathcal{L}_c^\dagger(\bm{\theta}_c)$. 
The estimator of $\bm{\theta}$ is obtained by maximizing $\mathcal{L}(\bm{\theta})$, regardless of whether administrative censoring is absent or present. 
Provided that the density model $f(t; \bm{\theta})$ is identifiable on $[0,\infty)$ in the absence of administrative censoring, or on $[0,\tau]$ in its presence, this estimator is asymptotically consistent. 
This property is useful when the primary focus is on the hazard of $T$, and that of $C$ is regarded as a nuisance. 
However, when the administrative censoring time $\tau$ is short, the identifiability condition can be restrictive for certain survival models.  
For instance, in piecewise-proportional hazards models \citep{Michael}, the baseline hazard parameters corresponding to intervals beyond $\tau$ are not identifiable.

\section{Survival Analysis under Right-Censored Reporting Delays} 
\label{approach} 
This section presents a general framework for survival analysis under right-censored reporting delays. 

\subsection{Problem Setting} 
\label{Setting} 
In this study, we consider a setting in which two events occur sequentially over time. 
The earlier and later events correspond to the event of interest and its reporting, respectively. 
We jointly model the time to the earlier event and the subsequent time from the earlier to the later event. 

\begin{wrapfigure}{r}{0.5\textwidth}
\vspace{-1em}
\centering
\includegraphics[width=0.47\textwidth]{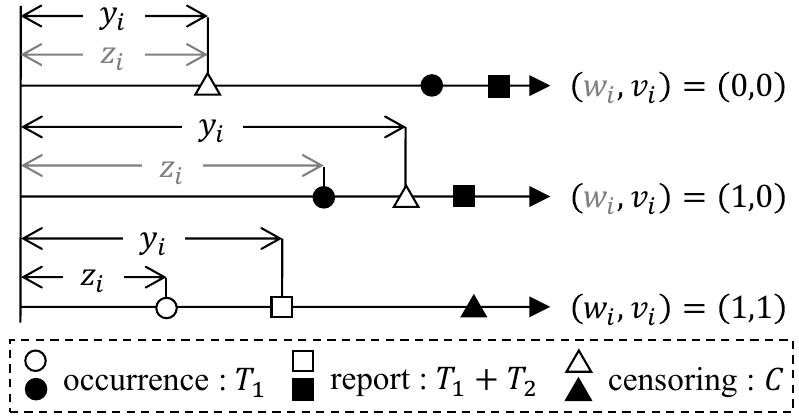}
\caption{
Notation for observations. 
The figure presents the notation for the three observation patterns shown in \Cref{overview}. 
In the top two patterns with $v_i = 0$, the pair $(z_i, w_i)$ remains unobserved. 
}
\label{pattern}
\end{wrapfigure}

We use subscripts $1$ and $2$ for the symbols $(T, h, f, S, \bm{\theta})$ to distinguish the earlier and later events, respectively.  
We assume that $T_1$, $T_2$, and $C$ are mutually independent conditional on $X$, and that $\bm{\theta}_1$, $\bm{\theta}_2$, and $\bm{\theta}_c$ are unknown. 
For the earlier event, we let $Z \coloneqq \min \{T_1 , C\}, W \coloneqq \mathbb{I}(T_1 \leq C), \bar{Z} \coloneqq \min\{Z, \tau\}$, and $\bar{W} \coloneqq \mathbb{I}(T_1 \leq C, T_1 \leq \tau)$.  
For the later event, we define $Y, V, \bar{Y}$, and $\bar{V}$ by substituting $T = T_1 + T_2$ into their respective expressions. 
Each observation is characterized by the covariates $\bm{x}_i \in \mathcal{X}$, 
the time $z_i \in [0, y_i]$ until the earlier event or censoring, the event indicator $w_i \in \{0, 1\}$ for the earlier event, 
the time $y_i$ until the later event or censoring, and the event indicator $v_i$ for the later event. 
For simplicity, the same notation $(z_i, y_i, w_i, v_i)$ is used to represent the realizations of both $(Z, Y, W, V)$ and $(\bar{Z}, \bar{Y}, \bar{W}, \bar{V})$. 
A similar setting involving two successive events under right censoring is also discussed by \citet{Visser}. 
We additionally consider the setting in which the earlier event is observed only if the later event occurs, as illustrated in \Cref{overview}.  
The triplet $(\bm{x}_i, y_i, v_i)$ is observed for $1 \leq i \leq n$, while the pair $(z_i, w_i)$ is observed only if $v_i = 1$.  
The probability of observing $(w, v) = (0, 1)$ is zero by the definitions of $(W, V)$ and $(\bar{W}, \bar{V})$. 
\Cref{pattern} illustrates how our notation corresponds to the three possible observation patterns under this extended setting. 
For simplicity, we set $v_i = 0$ for $1 \leq i \leq m$ and $v_i = 1$ for $m + 1 \leq i \leq n$.

\subsection{Estimator Formulation} 
\label{EF} 
We first construct the estimator of $(\bm{\theta}_1, \bm{\theta}_2)$ in the absence of administrative censoring. 
The corresponding density model can be derived through a straightforward calculation. 
\begin{lemma} 
\label{mixture} 
For $0 \leq y$, $v = 0$, and $\bm{x} \in \mathcal{X}$, the conditional joint distribution of $(Y, V)$ is given by  
\begin{align} 
\label{v0}  
g_{\scriptscriptstyle\blacksquare}(y \mid \bm{x}) \coloneqq \frac{\partial}{\partial y} \Pr(Y \leq y, V = 0 \mid X = \bm{x}) = S_{\scriptscriptstyle\blacksquare}(y \mid \bm{x}) f_c(y \mid \bm{x}), 
\end{align} 
where $S_{\scriptscriptstyle\blacksquare}(y \mid \bm{x}) \coloneqq S_1(y \mid \bm{x}) + \int_0^y f_1(t \mid \bm{x}) S_2(y - t \mid \bm{x}) dt$. 
For $0 \leq z \leq y$, $w = v = 1$, and $\bm{x} \in \mathcal{X}$, the conditional joint distribution of $(Z, Y, W, V)$ is given by  
\begin{align} 
g_{\scriptscriptstyle\square}(z, y \mid \bm{x}) \coloneqq \frac{\partial^2}{\partial y \partial z} \Pr(Z \leq z, Y \leq y, W = 1, V = 1 \mid X = \bm{x}) = f_{\scriptscriptstyle\square}(z, y \mid \bm{x}) S_c(y \mid \bm{x}), 
\end{align} 
where $f_{\scriptscriptstyle\square}(z, y \mid \bm{x}) \coloneqq f_1(z \mid \bm{x}) f_2(y - z \mid \bm{x})$. 
\end{lemma} 
In \Cref{mixture}, we specify the joint distribution only for the observed variables. 
Accordingly, we can write the likelihood for the observed data as follows. 
\begin{proposition} 
\label{loglikelihood_without}
The log likelihood in the absence of administrative censoring is $\mathcal{E}(\bm{\theta}_1, \bm{\theta}_2) + \mathcal{L}_c(\bm{\theta}_c)$, where $\mathcal{E}$ is defined as   
\begin{align} 
\label{likelihood_E} 
\mathcal{E}(\bm{\theta}_1, \bm{\theta}_2) \coloneqq \sum_{i = 1}^n \Bigl((1 - v_i) \log S_{\scriptscriptstyle\blacksquare}(y_i \mid \bm{x}_i) + v_i \log f_{\scriptscriptstyle\square}(z_i, y_i \mid \bm{x}_i)\Bigr). 
\end{align} 
\end{proposition} 
The parameters $(\bm{\theta}_1, \bm{\theta}_2)$ can be estimated by maximizing $\mathcal{E}(\bm{\theta}_1, \bm{\theta}_2)$, since $\mathcal{L}_c(\bm{\theta}_c)$ does not depend on $(\bm{\theta}_1, \bm{\theta}_2)$. 
Although the structure of the joint distribution resembles that of \Cref{mix_G}, it involves a newly introduced pair of functions, $(S_{\scriptscriptstyle\blacksquare}, f_{\scriptscriptstyle\square})$. 
In this case, the identifiability of $(f_1, f_2)$ with respect to $(\bm{\theta}_1, \bm{\theta}_2)$ is not straightforward, since $(Z, W)$ is observed only when $V = 1$. 
As a result, proving the consistency of the estimator is nontrivial. 
Nevertheless, it can be ensured under the standard identifiability condition. 
\begin{assumption} 
\label{f_identifiability}  
If $f_1(t \mid \bm{x}; \bm{\theta}_1) = f_1(t \mid \bm{x}; \bm{\theta}_1^\prime)$ and $f_2(t \mid \bm{x}; \bm{\theta}_2) = f_2(t \mid \bm{x}; \bm{\theta}_2^\prime)$ for Lebesgue-almost every $t \in [0, \infty)$ and $P_X$-almost every $\bm{x} \in \mathcal{X}$, then $(\bm{\theta}_1, \bm{\theta}_2) = (\bm{\theta}_1^\prime, \bm{\theta}_2^\prime)$. 
\end{assumption} 
\begin{theorem} 
\label{ind2}  
Under \Cref{f_identifiability}, the estimator of $(\bm{\theta}_1, \bm{\theta}_2)$ obtained by maximizing $\mathcal{E}(\bm{\theta}_1, \bm{\theta}_2)$ in the absence of administrative censoring is asymptotically consistent. 
\end{theorem} 

As in survival analysis without reporting delays, censoring can be ignored when estimating the event hazards. 
\Cref{APS} shows that this property remains valid even when the censoring times for the earlier and later events differ. 
There is another situation in which the hazard function of $C$ is also of interest, and $(\bm{\theta}_1, \bm{\theta}_2)$ shares some components with $\bm{\theta}_c$. 
Under the following assumption, we can jointly estimate $(\bm{\theta}_1, \bm{\theta}_2, \bm{\theta}_c)$. 
\begin{assumption} 
\label{c_identifiability}  
If $f_1(t \mid \bm{x}; \bm{\theta}_1) = f_1(t \mid \bm{x}; \bm{\theta}_1^\prime)$, $f_2(t \mid \bm{x}; \bm{\theta}_2) = f_2(t \mid \bm{x}; \bm{\theta}_2^\prime)$, and $f_c(t \mid \bm{x}; \bm{\theta}_c) = f_c(t \mid \bm{x}; \bm{\theta}_c^\prime)$ for Lebesgue-almost every $t \in [0, \infty)$ and $P_X$-almost every $\bm{x} \in \mathcal{X}$, then $(\bm{\theta}_1, \bm{\theta}_2, \bm{\theta}_c) = (\bm{\theta}_1^\prime, \bm{\theta}_2^\prime, \bm{\theta}_c^\prime)$. 
\end{assumption} 
\begin{proposition} 
\label{ind_all}  
Under \Cref{c_identifiability}, the estimator of $(\bm{\theta}_1, \bm{\theta}_2, \bm{\theta}_c)$ obtained by maximizing $\mathcal{E}(\bm{\theta}_1, \bm{\theta}_2) + \mathcal{L}_c(\bm{\theta}_c)$ in the absence of administrative censoring is asymptotically consistent. 
\end{proposition} 

Next, we consider the case with administrative censoring. 
The structure of the probability model differs slightly from that described in \Cref{mixture}. 
\begin{lemma}
\label{mix20} 
For $0 \leq y \leq \tau$, $v = 0$, and $\bm{x} \in \mathcal{X}$, the conditional joint distribution of $(\bar{Y}, \bar{V})$ is given by   
\begin{align}
\bar{g}_{\scriptscriptstyle\blacksquare}(y \mid \bm{x}) \coloneqq 
\begin{cases} 
\Pr(\bar{Y} = \tau, \bar{V} = 0 \mid X = \bm{x}) = S_{\scriptscriptstyle\blacksquare}(\tau \mid \bm{x}) S_c(\tau \mid \bm{x}) & (y = \tau), \\ 
\frac{\partial}{\partial y} \Pr(\bar{Y} \leq y, \bar{V} = 0 \mid X = \bm{x}) = S_{\scriptscriptstyle\blacksquare}(y \mid \bm{x}) f_c(y \mid \bm{x}) & (y \neq \tau). 
\end{cases} 
\end{align} 
For $0 \leq z \leq y \leq \tau$, $w = v = 1$, and $\bm{x} \in \mathcal{X}$, the conditional joint distribution of $(\bar{Z}, \bar{Y}, \bar{W}, \bar{V})$ is given by 
\begin{align} 
\bar{g}_{\scriptscriptstyle\square}(z, y \mid \bm{x}) \coloneqq \frac{\partial^2}{\partial y \partial z} \Pr(\bar{Z} \leq z, \bar{Y} \leq y, \bar{W} = 1, \bar{V} = 1 \mid X = \bm{x}) = f_{\scriptscriptstyle\square}(z, y \mid \bm{x}) S_c(y \mid \bm{x}).   
\end{align} 
\end{lemma} 
\begin{proposition} 
\label{loglikelihood_with} 
The log likelihood in the presence of administrative censoring is $\mathcal{E}(\bm{\theta}_1, \bm{\theta}_2) + \mathcal{L}_c^\dagger(\bm{\theta}_c)$. 
\end{proposition} 
The presence of administrative censoring results in the replacement of the density $f_c$ in \Cref{v0} at $y = \tau$ by the survival function $S_c$. 
Then the log likelihood factorization allows us to estimate the hazard functions of $T_1$ and $T_2$ without requiring the censoring time distribution. 
The following theorem shows that the consistency requires a more restrictive condition than \Cref{f_identifiability}. 
\begin{assumption} 
\label{known0} 
If $f_1(t_1 \mid \bm{x}; \bm{\theta}_1) f_2(t_2 \mid \bm{x}; \bm{\theta}_2) = f_1(t_1 \mid \bm{x}; \bm{\theta}_1^\prime) f_2(t_2 \mid \bm{x}; \bm{\theta}_2^\prime)$ for Lebesgue-almost every $t_1, t_2 \in [0, \tau]$ satisfying $t_1 + t_2 \leq \tau$ and $P_X$-almost every $\bm{x} \in \mathcal{X}$, then $(\bm{\theta}_1, \bm{\theta}_2) = (\bm{\theta}_1^\prime, \bm{\theta}_2^\prime)$. 
\end{assumption} 
\begin{theorem} 
\label{ind20}  
Under \Cref{known0}, the estimator of $(\bm{\theta}_1, \bm{\theta}_2)$ obtained by maximizing $\mathcal{E}(\bm{\theta}_1, \bm{\theta}_2)$ in the presence of administrative censoring is asymptotically consistent. 
\end{theorem} 

Because observations are restricted to the finite time interval $[0,\tau]$, we can no longer exploit the normalization property that a density integrates to one over its entire support. 
Consequently, \Cref{known0} cannot be decomposed into separate identifiability assumptions for $f_1$ and $f_2$, as in \Cref{f_identifiability}; instead, it is stated in terms of the identifiability of their product. 
This product-form requirement does not substantially change the nature of the identifiability difficulty: it mainly excludes special classes of hazard functions, such as the example in \Cref{example_hazard}. 
Rather, similarly to the setting without reporting delays, the essential difficulty of \Cref{known0} lies in establishing identifiability when the administrative censoring time $\tau$ is short. 
The transfer-learning approach proposed in \Cref{approach_transfer} is designed to address this difficulty induced by administrative censoring.

\subsection{Optimization} 
\label{JPE} 
One approach to jointly estimating the earlier and later hazard functions is to directly maximize $\mathcal{E}(\bm{\theta}_1, \bm{\theta}_2)$. 
However, computing its gradient with respect to the parameters is often challenging, as $\log S_{\scriptscriptstyle\blacksquare}(y \mid \bm{x})$ does not necessarily admit a closed-form expression. 
To overcome this problem, we maximize $\mathcal{E}(\bm{\theta}_1,\bm{\theta}_2)$ using a Monte Carlo expectation-maximization (EM) algorithm \citep{Tanner}. 
Specifically, we employ a lower bound on $\mathcal{E}(\bm{\theta}_1,\bm{\theta}_2)$ derived from Jensen's inequality \citep{Billingsley} and approximate the resulting expectations via Monte Carlo integration. 
This yields a tractable optimization procedure. 
We begin by establishing the following theorem, which forms the basis of an EM algorithm \citep{Dempster}. 
\begin{theorem} 
\label{EMbelow} 
Let $\mathcal{L}_1(\bm{\theta}_1 \mid \mathcal{D}_1)$ denote the log-likelihood obtained from \Cref{Likelihood} by setting $f = f_1$ and $\mathcal{D}=\mathcal{D}_1 \coloneqq (\bm{x}_i, z_i, w_i)_{i=1}^n$. 
Similarly, let $\mathcal{L}_2(\bm{\theta}_2 \mid \mathcal{D}_2)$ denote the corresponding log-likelihood obtained by setting $f = f_2$ and $\mathcal{D}=\mathcal{D}_2 \coloneqq (\bm{x}_i, y_i-z_i, v_i)_{i=1}^n$. 
Then $\mathcal{E}(\bm{\theta}_1, \bm{\theta}_2)$ is bounded below by  
\begin{align} 
\label{LB}
\mathcal{LB}(q, \bm{\theta}_1, \bm{\theta}_2) \coloneqq \mathcal{I} \bigl[\bigl(\mathcal{L}_1(\bm{\theta}_1 \mid \mathcal{D}_1) + \mathcal{L}_2(\bm{\theta}_2 \mid \mathcal{D}_2) - \log q\bigr) q\bigr], 
\end{align} 
where $\mathcal{I} \coloneqq \mathcal{I}_1 \circ \cdots \circ \mathcal{I}_m$, $\mathcal{I}_i [\,\cdot\,] \coloneqq [\,\cdot\,] \big|_{z_i = y_i, w_i = 0} + \int_{0}^{y_i} [\,\cdot\,] \big|_{w_i = 1} dz_i$, 
$q \coloneqq \prod_{i=1}^m q_i(z_i, w_i)$, and $q_i$ is a strictly positive function such that $\mathcal{I}_i[q_i(z_i, w_i)] = 1$.  
Given fixed parameters $(\bm{\theta}_1, \bm{\theta}_2)$, this lower bound on $\mathcal{E}(\bm{\theta}_1, \bm{\theta}_2)$ attains its maximum when, for each $1 \leq i \leq m$,   
\begin{align} 
\label{posterior} 
q_i(z_i, w_i) = 
\begin{cases} 
S_{\scriptscriptstyle\blacksquare}(y_i \mid \bm{x}_i; \bm{\theta}_1, \bm{\theta}_2)^{-1} S_1(z_i \mid \bm{x}_i; \bm{\theta}_1) & (w_i = 0), \\
S_{\scriptscriptstyle\blacksquare}(y_i \mid \bm{x}_i; \bm{\theta}_1, \bm{\theta}_2)^{-1} f_1(z_i \mid \bm{x}_i; \bm{\theta}_1) S_2(y_i - z_i \mid \bm{x}_i; \bm{\theta}_2) & (w_i = 1). 
\end{cases}  
\end{align}  
\end{theorem} 

We iteratively update the density $q$ and optimize the parameters $(\bm{\theta}_1, \bm{\theta}_2)$. 
Let $(q^{(\kappa)}, \bm{\theta}_1^{(\kappa)}, \bm{\theta}_2^{(\kappa)})$ denote the iterate after the $\kappa$-th iteration, starting from the initial values at $\kappa = 0$. 
In the $\kappa$-th iteration, we first set the density $q = q^{(\kappa)}$ as in \Cref{posterior}, which maximizes the lower bound $\mathcal{LB}(q, \bm{\theta}_1^{(\kappa - 1)}, \bm{\theta}_2^{(\kappa - 1)})$. 
We then obtain the updated parameters $(\bm{\theta}_1, \bm{\theta}_2) = (\bm{\theta}_1^{(\kappa)}, \bm{\theta}_2^{(\kappa)})$ by maximizing the lower bound $\mathcal{LB}(q^{(\kappa)}, \bm{\theta}_1, \bm{\theta}_2)$. 
In this iterative scheme, $\mathcal{E}(\bm{\theta}_1^{(\kappa)}, \bm{\theta}_2^{(\kappa)})$ is monotonically nondecreasing in $\kappa$. 

The integral in \Cref{LB} cannot be evaluated in closed form. 
To resolve this issue, we adopt a Monte Carlo approximation. 
The missing values in $\mathcal{D}_1$ and $\mathcal{D}_2$ are imputed using samples generated from $q$. 
This procedure is repeated $s$ times, and the resulting pseudo-complete datasets are then combined. 
Let $\mathcal{P}_1$ and $\mathcal{P}_2$ denote the combined pseudo-complete datasets corresponding to $\mathcal{D}_1$ and $\mathcal{D}_2$, respectively. 
\Cref{LB} is approximated by 
\begin{align} 
\label{original}
s^{-1} \mathcal{L}_1(\bm{\theta}_1 \mid \mathcal{P}_1) + s^{-1} \mathcal{L}_2(\bm{\theta}_2 \mid \mathcal{P}_2) - \mathcal{I} [q \log q]. 
\end{align} 
The formulation in \Cref{original} enables the independent optimization of $\bm{\theta}_1$ and $\bm{\theta}_2$ using standard full-likelihood methods, thereby facilitating the reuse of existing resources. 

The remaining problem is how to efficiently generate samples from $q$ for use in Monte Carlo integration. 
As stated at the end of \Cref{EMbelow}, the normalization constant $S_{\scriptscriptstyle\blacksquare}(y_i \mid \bm{x}_i; \bm{\theta}_1, \bm{\theta}_2)$ appears in the sampling distribution $q_i(z_i, w_i)$. 
To avoid computing this constant explicitly, we adopt the rejection sampling scheme:  
For each $1 \leq i \leq m$, repeatedly draw $t \sim f_1(t \mid \bm{x}_i)$ until one of the following conditions is met. 
If the sampled earlier-event time $t$ exceeds the censoring time $y_i$, we set $(z_i, w_i) = (y_i, 0)$. 
Otherwise, we generate an independent $U \sim \mathrm{Uniform}(0,1)$. 
If $U \leq S_2(y_i - t \mid \bm{x}_i)$, we accept the sample and set $(z_i, w_i) = (t, 1)$. 
If $U > S_2(y_i - t \mid \bm{x}_i)$, we reject the sampled $t$ and redraw.

\section{Transfer Learning for Administratively Censored Reporting Delays} 
\label{approach_transfer} 
As discussed in \Cref{EF}, asymptotic consistency of the parameter estimator can be established under a moderate assumption in the absence of administrative censoring. 
In contrast, under administrative censoring, consistency requires a restrictive identifiability condition. 
This section presents a transfer-learning procedure designed to address this issue. 

\subsection{Transfer Setting} 
\label{transfer_basic} 
We consider a setting in which the target domain is subject to administrative censoring, and supplementary observations unaffected by administrative censoring are available from the source domain. 
We assume that the observations from the source and target domains are independent. 
We introduce a two-stage estimation procedure by decomposing the parameter vectors in the earlier and later hazards as $\bm{\theta}_1 = (\bm{\theta}_{1,1}, \bm{\theta}_{1,2})$ and $\bm{\theta}_2 = (\bm{\theta}_{2,1}, \bm{\theta}_{2,2})$, respectively.  
Let $\mathcal{E}_s$ and $\mathcal{E}_t$ be defined as $\mathcal{E}$ evaluated on the observations from the source and target domains, respectively.  
In the first stage, we jointly estimate $(\bm{\theta}_1, \bm{\theta}_2)$ by maximizing $\mathcal{E}_s(\bm{\theta}_1, \bm{\theta}_2)$, and denote the corresponding first-stage estimator of $(\bm{\theta}_{1,1}, \bm{\theta}_{2,1})$ by $(\hat{\bm{\theta}}_{1,1}, \hat{\bm{\theta}}_{2,1})$.   
In the second stage, we estimate the remaining components $(\bm{\theta}_{1,2}, \bm{\theta}_{2,2})$ by maximizing $\mathcal{E}_t((\hat{\bm{\theta}}_{1,1}, \bm{\theta}_{1,2}), (\hat{\bm{\theta}}_{2,1}, \bm{\theta}_{2,2}))$, and denote the resulting estimator by $(\hat{\bm{\theta}}_{1,2}, \hat{\bm{\theta}}_{2,2})$. 
Both stages can be carried out by applying the procedure described in \Cref{JPE}. 

We next establish the asymptotic consistency of the estimator obtained by the above two-stage procedure.  
This consistency relies on the following transfer condition. 
\begin{assumption} 
\label{transfer_condition}  
The true values of the transferable components $(\bm{\theta}_{1,1}, \bm{\theta}_{2,1})$, denoted by $(\bm{\theta}_{1,1}^*, \bm{\theta}_{2,1}^*)$, are shared by the source and target domains.  
Moreover, the remaining components $(\bm{\theta}_{1,2}, \bm{\theta}_{2,2})$ are identifiable in the following sense: 
if $f_1(t_1 \mid \bm{x}; (\bm{\theta}_{1,1}^*, \bm{\theta}_{1,2})) f_2(t_2 \mid \bm{x}; (\bm{\theta}_{2,1}^*, \bm{\theta}_{2,2})) = f_1(t_1 \mid \bm{x}; (\bm{\theta}_{1,1}^*, \bm{\theta}_{1,2}^\prime)) f_2(t_2 \mid \bm{x}; (\bm{\theta}_{2,1}^*, \bm{\theta}_{2,2}^\prime))$ for Lebesgue-almost every $t_1, t_2 \in [0, \tau]$ satisfying $t_1 + t_2 \leq \tau$ and $P_X$-almost every $\bm{x} \in \mathcal{X}$, then $(\bm{\theta}_{1,2}, \bm{\theta}_{2,2}) = (\bm{\theta}_{1,2}^\prime, \bm{\theta}_{2,2}^\prime)$. 
\end{assumption} 
Intuitively, \Cref{transfer_condition} means that \Cref{known0} holds for the remaining components $(\bm{\theta}_{1,2}, \bm{\theta}_{2,2})$ once the transferable components $(\bm{\theta}_{1,1}, \bm{\theta}_{2,1})$ are known.  
In the two-stage procedure, these true values are replaced by their first-stage estimates.  
The following result formalizes this intuition. 
\begin{theorem} 
\label{two_stage_consistency} 
Let $n_s$ and $n_t$ denote the sample sizes in the source and target domains, respectively. 
Then, under \Cref{f_identifiability,transfer_condition}, the two-stage estimator $((\hat{\bm{\theta}}_{1,1}, \hat{\bm{\theta}}_{1,2}), (\hat{\bm{\theta}}_{2,1}, \hat{\bm{\theta}}_{2,2}))$ is asymptotically consistent as $n_s, n_t \to \infty$. 
\end{theorem} 
While the two-stage procedure ensures consistency under administrative censoring, the first-stage estimation error propagates into the second stage and affects the asymptotic covariance of the resulting estimator. 
The following proposition quantifies this effect. 
\begin{proposition} 
\label{variance_proposition} 
Let $(\bm{\theta}_{1,2}^*, \bm{\theta}_{2,2}^*)$ denote the target-domain true values of $(\bm{\theta}_{1,2}, \bm{\theta}_{2,2})$. 
Then, under \Cref{f_identifiability,transfer_condition}, the covariance matrix of $(\hat{\bm{\theta}}_{1,2}, \hat{\bm{\theta}}_{2,2}) - (\bm{\theta}_{1,2}^*, \bm{\theta}_{2,2}^*)$ admits the expansion   
\begin{align} 
\label{covariance_matrix}
n_t^{-1} I_1^{-1} + n_s^{-1} I_1^{-1} I_2 I_3^{-1} I_2^\top I_1^{-1} + o(n_s^{-1} + n_t^{-1}),  
\end{align} 
where $I_1$ is the Fisher information matrix for the remaining components $(\bm{\theta}_{1,2}, \bm{\theta}_{2,2})$ in the target-domain objective $n_t^{-1} \mathcal{E}_t(\bm{\theta}_1,\bm{\theta}_2)$, 
$I_2$ is the corresponding cross-information matrix between the remaining components $(\bm{\theta}_{1,2}, \bm{\theta}_{2,2})$ and the transferable components $(\bm{\theta}_{1,1}, \bm{\theta}_{2,1})$ in the target-domain objective, 
and $I_3$ is the Schur-complement information matrix for the transferable components $(\bm{\theta}_{1,1}, \bm{\theta}_{2,1})$ in the source-domain objective $n_s^{-1} \mathcal{E}_s(\bm{\theta}_1,\bm{\theta}_2)$. 
All matrices are evaluated at the true parameter values. 
\end{proposition} 

We illustrate the above transfer-learning procedure using a parametric proportional hazards model for the earlier event time $T_1$. 
To allow a more general specification, we write the earlier hazard as $h_1(t \mid \bm{x}; \bm{\theta}_1) = h_b(t \mid \bm{x}; \bm{\alpha}) \phi(\bm{x}; \bm{\beta})$, where the baseline component $h_b$ is allowed to depend on the covariates. 
As a simple setting in which \Cref{transfer_condition} can be verified, we assume that the baseline component $h_b$ and the later hazard $h_2$ are shared by the source and target domains, whereas the time-independent component $\phi$ in the earlier hazard $h_1$ is domain-specific. 
Accordingly, we set $\bm{\theta}_{1,1} = \bm{\alpha}$, $\bm{\theta}_{1,2} = \bm{\beta}$, $\bm{\theta}_{2,1} = \bm{\theta}_2$, and $\bm{\theta}_{2,2} = \emptyset$.  
Under this assumption, \Cref{transfer_condition} is satisfied provided that the domain-specific parameter $\bm{\beta}$ is identifiable. 
This setting is suitable when domains differ in disease occurrence levels through $\phi$ but share the onset-time pattern $h_b$ and the post-onset reporting process $h_2$.

\subsection{Optimization for Parametric Proportional Hazards Models} 
\label{transfer_PPH} 
The procedure described in \Cref{JPE} can be used for parameter estimation in the second stage.  
However, the repeated Monte Carlo approximation required by this procedure may induce substantial fluctuations in the second-stage estimates at each update as the administrative censoring time $\tau$ increases and new target-domain data accumulate. 
Such instability is undesirable for operational decision-making, where practitioners require stable estimates to monitor risk and adapt their actions over time. 
We next show that, in the transfer-learning setting described at the end of \Cref{transfer_basic}, this issue can be alleviated by constructing a more tractable estimating equation. 
For clarity, we denote the target-domain coefficient by $\bm{\gamma}$ instead of $\bm{\beta}$, while writing the transferred source-domain estimates as $\hat{\bm{\alpha}}$ and $\hat{\bm{\theta}}_2$.  
In the following, $(\bm{x}_i, z_i, y_i, w_i, v_i)_{i=1}^n$ denotes the target-domain observations. 
We further assume that the baseline hazard function $h_b$ is piecewise continuous in time and bounded on every compact time interval. 

We aim to construct a more tractable approximation to the original estimating equation $\bm{\zeta}_0(\bm{\gamma}) = \bm{0}$, 
where $\bm{\zeta}_0(\bm{\gamma}) \coloneqq \frac{\partial}{\partial \bm{\gamma}} \mathcal{E}_t((\hat{\bm{\alpha}}, \bm{\gamma}), \hat{\bm{\theta}}_2)$. 
Let $\hat{\bm{\gamma}}$ denote a solution to this equation. 
The evaluation of $\bm{\zeta}_0(\bm{\gamma})$ involves $S_{\scriptscriptstyle\blacksquare}(y \mid \bm{x})$, whose second term may not admit a closed-form expression. 
As a simple approximation, we ignore this second term. 
Then the equation $\bm{\zeta}_0(\bm{\gamma}) = \bm{0}$ becomes 
\begin{align} 
\label{check0}
\bm{\zeta}(\bm{\gamma}) \coloneqq \sum_{i = 1}^n \Bigl(v_i \frac{\partial \log \phi(\bm{x}_i; \bm{\gamma})}{\partial \bm{\gamma}} - r_i \frac{\partial \phi(\bm{x}_i; \bm{\gamma})}{\partial \bm{\gamma}}\Bigr) = \bm{0}, 
\end{align} 
where $r_i \coloneqq \int_0^{\tilde{y}_i} h_b(t \mid \bm{x}_i; \hat{\bm{\alpha}}) dt$, with $\tilde{y}_i \coloneqq y_i$ for $1 \leq i \leq m$ and $\tilde{y}_i \coloneqq z_i$ for $m < i \leq n$. 
Although solving \Cref{check0} often yields stable estimates, it can lead to substantial bias due to the omitted term; see \Cref{bias}. 
To mitigate this bias, we focus on the regime in which only a small proportion of individuals experience the event over $[0,\tau]$. 
This regime is plausible when the early event hazard is low or the administrative censoring time $\tau$ is short. 
In this regime, we derive a more accurate approximation to the original estimating equation $\bm{\zeta}_0(\bm{\gamma}) = \bm{0}$ by applying a Taylor expansion to the difference $\bm{\zeta}_0(\hat{\bm{\gamma}}) - \bm{\zeta}(\hat{\bm{\gamma}})$. 
\begin{assumption} 
\label{ynu}  
Let $\nu_i \coloneqq \phi(\bm{x}_i; \hat{\bm{\gamma}}) \sup_{t \in [0, y_i]} h_b(t \mid \bm{x}_i; \hat{\bm{\alpha}})$. 
Assume that $\nu_i y_i$ is bounded above by a common finite constant for all $1 \leq i \leq m$. 
\end{assumption} 
\begin{theorem} 
\label{MLE} 
Let $\tilde{\bm{\zeta}}$ be defined analogously to $\bm{\zeta}$ by replacing $r_i$ with 
\begin{align} 
\tilde{r}_i \coloneqq \int_0^{\tilde{y}_i} h_b(t \mid \bm{x}_i; \hat{\bm{\alpha}}) \bigl(1 - S_2(y_i - t \mid \bm{x}_i; \hat{\bm{\theta}}_2)\bigr)^{1 - v_i} dt. 
\end{align} 
Let $\delta_i$ denote the fixed point satisfying  
\begin{align} 
\label{equation_delta} 
\delta_i = \int_{\delta_i}^{y_i} \exp\Bigl(\int_0^t \bigl(\nu_i - h_2(u \mid \bm{x}_i; \hat{\bm{\theta}}_2)\bigr) du\Bigr) dt. 
\end{align} 
Let $\bm{\pi}_i$ denote the vector of componentwise absolute values of $\frac{\partial}{\partial \bm{\gamma}} \log \phi(\bm{x}_i; \bm{\gamma}) |_{\bm{\gamma} = \hat{\bm{\gamma}}}$.  
Then, under \Cref{ynu}, we have $\tilde{\bm{\zeta}}(\hat{\bm{\gamma}}) = \mathcal{O}(\sum_{i = 1}^m \nu_i^2 \delta_i y_i \bm{\pi}_i)$. 
\end{theorem} 

\Cref{MLE} shows that $\tilde{\bm{\zeta}}(\bm{\gamma}) = \bm{0}$ can be a reasonable approximation of the original estimating equation $\bm{\zeta}_0(\bm{\gamma}) = \bm{0}$. 
Note that the proof provides the multiplicative constant in the Landau notation. 
For $v_i = 0$, the coefficient $\tilde{r}_i$ can be viewed as a weighted version of $r_i$, where the weight represents the probability that an event occurring at time $t$ is reported within the remaining observation period. 
This weight arises from the approximation of $\bm{\zeta}_0(\hat{\bm{\gamma}}) - \bm{\zeta}(\hat{\bm{\gamma}})$. 
Aside from a one-time precomputation of $\tilde{r}_i$, solving the approximate equation avoids repeated numerical integration and can therefore improve numerical stability compared with the procedure described in \Cref{JPE}. 

We can see that $\tilde{\bm{\zeta}}(\hat{\bm{\gamma}})$ has a smaller error when the later hazard $h_2$ is substantially larger than the earlier hazard $h_1$. 
This regime can arise, for example, when disease onset is rare but subsequent reporting, such as death after onset, occurs relatively frequently. 
In this regime, the right-hand side of \Cref{equation_delta} is strictly decreasing and rapidly decays to zero as $\delta_i$ increases. 
Therefore, the fixed-point solution of \Cref{equation_delta} is close to zero. 
Consequently, the equation $\tilde{\bm{\zeta}}(\bm{\gamma}) = \bm{0}$ provides a reasonable approximation to the original estimating equation $\bm{\zeta}_0(\bm{\gamma}) = \bm{0}$.

\section{Numerical Study} 
\label{Experiment} 
In this section, we present experimental evidence supporting the validity of our approach. 
The experiment is motivated by timely risk evaluation for individuals newly enrolled in an at-risk monitoring program, where disease-onset times are retrospectively observed only at death. 

\textbf{Survival Model.}  
The earlier hazard is modeled by a piecewise-proportional hazards model. 
Specifically, the baseline hazard is defined as $h_b(t; \bm{\alpha}) = \alpha_1 \mathbb{I}(t \in [0, 0.5]) + \alpha_2 \mathbb{I}(t \in (0.5, 1]) + \alpha_3 \mathbb{I}(t \in (1, \infty))$, where $\bm{\alpha} = (\alpha_i)_{i=1}^3 \in (0, \infty)^3$. 
The earlier hazards in the source and target domains are modeled as $h_b(t; \bm{\alpha}) \exp(\bm{\beta}^\top \bm{x})$ and $h_b(t; \bm{\alpha}) \exp(\bm{\gamma}^\top \bm{x})$, respectively, 
where $\bm{\beta} = (\beta_i)_{i=1}^d \in \mathbb{R}^d$ and $\bm{\gamma} = (\gamma_i)_{i=1}^d \in \mathbb{R}^d$. 
The later hazard is assumed to be constant over time, with $h_2(t) = \lambda \in (0,\infty)$. 

\textbf{Estimation.} 
The parameters for the source domain were estimated using the procedure described in \Cref{JPE}, with $100$ iterations and $s = 10$. 
Increasing these values had little effect on the resulting estimates. 
The maximization of $\mathcal{L}_1(\bm{\theta}_1 \mid \mathcal{P}_1)$ was performed using the lifelines library, version 0.27.8, with the penalizer parameter fixed at $10^{-4}$ for numerical stability.  
All other hyperparameters in lifelines were set to their default values. 
The parameters were initialized using estimates obtained by substituting the missing values $(z_i, w_i)_{i=1}^m$ with $(y_i, v_i)_{i=1}^m$. 
For the target domain, we computed $\bm{\gamma}$ by solving the approximate estimating equation $\tilde{\bm{\zeta}}(\bm{\gamma}) = \bm{0}$ via an L-BFGS-based minimization procedure, initialized at $\bm{\gamma} = \bm{0}$. 
The L-BFGS optimization was implemented using SciPy, version 1.2.1, with default solver options. 
The computational environment is provided in \Cref{CE}.  
We refer to this transfer-based estimation procedure as \textit{Ours}. 

\textbf{Baselines.} 
We compare our approach with \textit{Standard}, a survival model that does not explicitly model reporting delays and is characterized only by the event hazard $h_1$.  
This baseline uses the same transfer procedure as \textit{Ours}, except that the missing values $(z_i, w_i)_{i=1}^m$ are replaced with $(y_i, v_i)_{i=1}^m$ in the relevant domains. 
For both \textit{Standard} and \textit{Ours}, we also consider variants that learn directly from the target observations, denoted by \textit{Standard Direct} and \textit{Ours Direct}, respectively. 
To satisfy the identifiability conditions, we impose $\alpha_1 = \alpha_2 = \alpha_3$ when $\tau \leq 0.5$, and $\alpha_2 = \alpha_3$ when $0.5 < \tau \leq 1$. 
Moreover, we consider a variant of \textit{Ours} that estimates the target-domain parameter $\bm{\gamma}$ using the procedure described in \Cref{JPE}, with $100$ iterations and $s = 10$, instead of the approximate estimating equation. 
We refer to this variant as \textit{Ours MCEM}. 

\textbf{Toy Data.} 
For each of the source and target domains, we generated $1000$ observations with covariate dimension $d = 10$ from the survival model specified by the following parameters. 
The value of $(\alpha_1, \alpha_2, \alpha_3, \beta_1, \gamma_1)$ was set to $(0.1, 0.2, 0.3, 1, 2)$. 
For each $2 \leq i \leq d$, $\beta_i$ and $\gamma_i$ were independently drawn from the uniform distribution on $[-1,1]$. 
To assess the impact of reporting delays, we considered two values of $\lambda$, $1$ and $5$, denoted by \textit{Toy1} and \textit{Toy2}, respectively. 
The censoring time was generated from a distribution with constant hazard rate $1$. 
For both the source and target domains, each covariate vector was generated uniformly over $[-1,1]^d$. 

\textbf{Real Data.} 
We adopted two benchmark datasets, \texttt{Colon} and \texttt{Rotterdam}, from the Rdatasets repository, version 1.0.0. 
These datasets contain 929 and 2982 individuals, respectively, and provide time-to-event information on disease recurrence and death. 
To introduce right-censored reporting delays, we treat death occurring after disease recurrence as the reporting event. 
We rescale $z_i$ and $y_i$ by the median observation time within each dataset.  
Both datasets contain the \texttt{age} variable. 
For each dataset, the youngest $25\%$ of individuals are assigned to the target domain, whereas the remaining $75\%$ form the source domain. 
We refer to these settings as \textit{Age1} and \textit{Age2}, respectively. 
We also consider an alternative split for each dataset.  
For the \texttt{Colon} dataset, we define the target domain using the \texttt{surg} variable, which indicates whether the interval from surgery to registration is short or long. 
Patients with long intervals are assigned to the registration-delayed target cohort, whereas those with shorter intervals form the source domain. 
For the \texttt{Rotterdam} dataset, we define the target domain using the \texttt{year} variable, which indicates the year of surgery. 
After ordering patients by surgery year, the latest $25\%$ are assigned to the temporally later target cohort, whereas the remaining $75\%$ form the source domain. 
We refer to these settings as \textit{Surg} and \textit{Year}, respectively. 
Once the split-defining variables are excluded, the \texttt{Colon} and \texttt{Rotterdam} datasets contain $10$ and $9$ covariates, respectively. 
These covariates were used in the survival models, and the preprocessing steps are described in \Cref{DP}. 

\textbf{Evaluation.} 
We evaluate the estimation performance under varying administrative censoring times $\tau \in \{0.25, 0.5, 0.75, 1\}$. 
Because $\tau \leq 1$ in all cases, identifiability of the survival model cannot be ensured from the target-domain observations alone, which motivates the transfer-learning procedure. 
We conducted $100$ trials. 
In each trial, a randomly selected half of the source-domain observations was used to estimate the parameters in the first stage of the two-stage procedure. 
The target-domain observations were randomly shuffled and evenly split into two subsets. 
One subset was administratively censored at each $\tau$ and used for parameter estimation in the second stage. 
The other subset served as a hold-out set to evaluate the estimated disease-onset hazard using the integrated Brier score \citep{Erika}. 
The evaluation horizon was set to the larger of the maximum onset time and the maximum censoring time in the hold-out set. 
For inverse probability of censoring weighting, we estimated the censoring survival function using the Kaplan--Meier estimator fitted on the hold-out set. 

\textbf{Result.} 
\Cref{result_toy_real} shows that \textit{Ours} performs competitively with the baselines and achieves better performance in many settings. 
These results support the effectiveness of both modeling reporting delays and leveraging transfer learning. 
For $\tau = 0.25$, \textit{Ours} consistently outperforms \textit{Ours MCEM}, supporting the advantage of using the approximate estimating equation. 
The parameter estimation results are reported in \Cref{AR}. 

\begin{table*}[t]
\caption{
Performance of risk evaluation. 
The mean and standard deviation of the integrated Brier score over $100$ trials are reported; lower values indicate better performance. 
} 
\label{result_toy_real}
\begin{center}
\begin{tabular}{cccccc} 
\toprule
Dataset & Method & $\tau = 0.25$ & $\tau = 0.5$ & $\tau = 0.75$ & $\tau = 1$ \\ 
\midrule 
   & Standard Direct & $0.407\ (0.051)$ & $0.342\ (0.043)$ & $0.334\ (0.044)$ & $0.284\ (0.036)$ \\ 
   & Standard & $0.221\ (0.030)$ & $0.198\ (0.027)$ & $0.187\ (0.025)$ & $0.180\ (0.023)$ \\ 
Toy1 & Ours Direct & $0.342\ (0.087)$ & $0.238\ (0.062)$ & $0.176\ (0.057)$ & $0.155\ (0.034)$ \\ 
   & Ours MCEM & $0.187\ (0.044)$ & $0.149\ (0.031)$ & $0.137\ (0.025)$ & $0.132\ (0.024)$ \\ 
   & Ours & $0.179\ (0.043)$ & $0.147\ (0.031)$ & $0.136\ (0.024)$ & $0.132\ (0.022)$ \\ 
\midrule 
   & Standard Direct & $0.264\ (0.040)$ & $0.211\ (0.031)$ & $0.186\ (0.031)$ & $0.154\ (0.023)$ \\ 
   & Standard & $0.143\ (0.021)$ & $0.131\ (0.019)$ & $0.128\ (0.018)$ & $0.127\ (0.019)$ \\ 
Toy2 & Ours Direct & $0.199\ (0.039)$ & $0.170\ (0.028)$ & $0.137\ (0.022)$ & $0.132\ (0.020)$ \\ 
   & Ours MCEM & $0.142\ (0.025)$ & $0.130\ (0.021)$ & $0.127\ (0.020)$ & $0.125\ (0.021)$ \\ 
   & Ours & $0.140\ (0.024)$ & $0.129\ (0.021)$ & $0.126\ (0.019)$ & $0.125\ (0.020)$ \\ 
\midrule 
   & Standard Direct & $0.298\ (0.037)$ & $0.252\ (0.022)$ & $0.246\ (0.017)$ & $0.237\ (0.016)$ \\ 
   & Standard & $0.320\ (0.035)$ & $0.272\ (0.024)$ & $0.245\ (0.018)$ & $0.238\ (0.016)$ \\ 
Surg & Ours Direct & $0.289\ (0.035)$ & $0.250\ (0.023)$ & $0.240\ (0.018)$ & $0.235\ (0.016)$ \\ 
   & Ours MCEM & $0.291\ (0.037)$ & $0.250\ (0.023)$ & $0.239\ (0.018)$ & $0.234\ (0.016)$ \\ 
   & Ours & $0.276\ (0.034)$ & $0.246\ (0.021)$ & $0.237\ (0.017)$ & $0.233\ (0.016)$ \\ 
\midrule 
   & Standard Direct & $0.300\ (0.040)$ & $0.253\ (0.028)$ & $0.238\ (0.022)$ & $0.229\ (0.019)$  \\ 
   & Standard & $0.323\ (0.040)$ & $0.269\ (0.031)$ & $0.235\ (0.021)$ & $0.229\ (0.019)$  \\ 
Age1 & Ours Direct & $0.299\ (0.046)$ & $0.260\ (0.029)$ & $0.232\ (0.024)$ & $0.230\ (0.020)$  \\ 
   & Ours MCEM & $0.292\ (0.046)$ & $0.249\ (0.032)$ & $0.232\ (0.023)$ & $0.228\ (0.019)$  \\ 
   & Ours & $0.280\ (0.042)$ & $0.246\ (0.029)$ & $0.230\ (0.020)$ & $0.227\ (0.019)$  \\ 
\midrule 
   & Standard Direct & $0.238\ (0.017)$ & $0.200\ (0.012)$ & $0.199\ (0.011)$ & $0.192\ (0.010)$  \\ 
   & Standard & $0.249\ (0.017)$ & $0.213\ (0.013)$ & $0.195\ (0.011)$ & $0.191\ (0.010)$  \\ 
Year & Ours Direct & $0.213\ (0.019)$ & $0.180\ (0.010)$ & $0.185\ (0.009)$ & $0.177\ (0.009)$  \\ 
   & Ours MCEM & $0.209\ (0.019)$ & $0.181\ (0.010)$ & $0.174\ (0.008)$ & $0.172\ (0.008)$  \\ 
   & Ours & $0.204\ (0.017)$ & $0.181\ (0.009)$ & $0.175\ (0.008)$ & $0.173\ (0.008)$  \\ 
\midrule 
   & Standard Direct & $0.370\ (0.028)$ & $0.277\ (0.018)$ & $0.289\ (0.016)$ & $0.268\ (0.013)$  \\ 
   & Standard & $0.404\ (0.022)$ & $0.319\ (0.017)$ & $0.279\ (0.016)$ & $0.264\ (0.014)$  \\ 
Age2 & Ours Direct & $0.334\ (0.041)$ & $0.253\ (0.019)$ & $0.242\ (0.017)$ & $0.240\ (0.013)$  \\ 
   & Ours MCEM & $0.318\ (0.042)$ & $0.250\ (0.018)$ & $0.233\ (0.015)$ & $0.229\ (0.013)$  \\ 
   & Ours & $0.311\ (0.039)$ & $0.251\ (0.017)$ & $0.233\ (0.014)$ & $0.228\ (0.013)$  \\ 
\bottomrule 
\end{tabular}
\end{center}
\end{table*}

\section{Conclusion}
\label{Conclusion} 
This study addressed survival modeling with right-censored reporting delays and proposed a transfer-learning procedure to mitigate the identifiability challenges posed by administrative censoring. 
The setting considered in this study is not limited to biostatistical applications. 
Future work should further examine the validity of the transfer assumptions in application-specific settings, since misspecified shared components may lead to biased risk estimates and inappropriate decision-making.

\medskip
\small


\bibliographystyle{unsrtnat}
\bibliography{neurips2026_LSM}

@article{Tanner, 
author =  "Greg C. G. Wei and Martin A. Tanner", 
title = "A {M}onte {C}arlo Implementation of the {EM} Algorithm and the Poor Man's Data Augmentation Algorithms",
journal = "Journal of the American Statistical Association",
volume =  "85",
number = "411",
year = "1990",
pages = "699--704"}

@Book{Billingsley, 
author = "Patrick Billingsley", 
title =  "Probability and Measure", 
edition   = "Third",
publisher = "Wiley", 
year =  "1995" 
}

@article{Anastasios, 
author =  "Ping Hu and Anastasios A. Tsiatis", 
title = "Estimating the Survival Distribution When Ascertainment of Vital Status is Subject to Delay",
journal = "Biometrika",
volume =  "83",
number = "2",
year = "1996",
pages = "371--380"}

@article{Hubbard, 
author =  "Mark J. Van Der Laan and Alan E. Hubbard", 
title = "Locally efficient estimation of the survival distribution with right-censored data and covariates when collection of data is delayed",
journal = "Biometrika",
volume =  "85",
number = "4",
year = "1998",
pages = "771--783"}

@article{Visser, 
author =  "Michael Visser", 
title = "Nonparametric Estimation of the Bivariate Survival Function with an Application to Vertically Transmitted {AIDS}",
journal = "Biometrika",
volume =  "83",
number = "3",
year = "1996",
pages = "507--518"}

@article{Lagakos, 
author =  "Stephen W. Lagakos and Leila M. Barraj and V. De Gruttola", 
title = "Nonparametric Analysis of Truncated Survival Data, with Application to {AIDS}",
journal = "Biometrika",
volume =  "75",
number = "3",
year = "1988",
pages = "515--523"}

@article{Lawless0, 
author =  "John D. Kalbfleisch and Jerald F. Lawless", 
title = "Inference Based on Retrospective Ascertainment: An Analysis of the Data on Transfusion-Related {AIDS}",
journal = "Journal of the American Statistical Association",
volume =  "84",
number = "406",
year = "1989",
pages = "360--372"}

@article{Harris, 
author =  "Jeffrey E. Harris", 
title = "Reporting Delays and the Incidence of {AIDS}",
journal = "Journal of the American Statistical Association",
volume =  "85",
number = "412",
year = "1990",
pages = "915--924"}

@article{Lawless, 
author =  "John D. Kalbfleisch and Jerald F. Lawless", 
title = "Regression models for right truncated data with application to {AIDS} incubation times and reporting lags",
journal = "Statistica Sinica",
volume =  "1",
year = "1991",
pages = "19--32"}

@article{Marcello, 
author =  "Marcello Pagano and Xin Ming Tu and Victor De Gruttola and Samantha MaWhinney", 
title = "Regression Analysis of Censored and Truncated Data: Estimating Reporting-Delay Distributions and {AIDS} Incidence from Surveillance Data",
journal = "Biometrics",
volume =  "50",
number = "4",
year = "1994",
pages = "1203--1214"}

@article{Matthias, 
author =  "Michael Höhle and Matthias an der Heiden", 
title = "{B}ayesian Nowcasting during the {STEC} {O104:H4} Outbreak in {G}ermany, 2011",
journal = "Biometrics",
volume =  "70",
number = "4",
year = "2014",
pages = "993--1002"}

@article{Wood, 
author =  "Simon N. Wood", 
title = "Inferring {UK} {COVID}-19 Fatal Infection Trajectories from Daily Mortality Data: Were Infections Already in Decline Before the {UK} Lockdowns?",
journal = "Biometrics",
volume =  "78",
number = "3",
year = "2022",
pages = "1127--1140"}

@article{Stoner1, 
author =  "Oliver Stoner and Theo Economou", 
title = "Multivariate hierarchical frameworks for modeling delayed reporting in count data",
journal = "Biometrics",
volume =  "76",
number = "3",
year = "2020",
pages = "789--798"}

@article{Stoner2, 
author =  "Oliver Stoner and Alba Halliday and Theo Economou", 
title = "Correcting delayed reporting of {COVID}-19 using thegeneralized-Dirichlet-multinomial method",
journal = "Biometrics",
volume =  "79",
number = "3",
year = "2023",
pages = "2537--2550"}

@article{Michael, 
author =  "Michael Friedman",
title = "Piecewise Exponential Models for Survival Data with Covariates",
journal = "The Annals of Statistics",
volume =  "10",
number = "1",
year = "1982",
pages = "101--113"}

@article{Holford, 
author =  "Theodore R. Holford",
title = "The Analysis of Rates and of Survivorship Using Log-Linear Models",
journal = "Biometrics",
volume =  "36",
number = "2",
year = "1980",
pages = "299--305"}

@article{Patrick, 
author =  "Patrick Royston and Mahesh Parmar",
title = "Flexible Parametric Proportional-Hazards and Proportional-Odds Models for Censored Survival Data, with Application to Prognostic Modelling and Estimation of Treatment Effects",
journal = "Statistics in Medicine",
volume =  "21",
number = "15",
year = "2002",
pages = "2175--2197"}

@article{Cox0, 
author =  "D. R. Cox",
title = "Regression Models and Life-Tables",
journal = "Journal of the Royal Statistical Society. Series B (Methodological)", 
volume =  "34",
number = "2",
year = "1972",
pages = "187--220"}

@article{Cox1, 
author =  "D. R. Cox",
title = "Partial Likelihood",
journal = "Biometrika", 
volume =  "62",
number = "2", 
year = "1975",
pages = "269--276"}

@article{Breslow0, 
author =  "Norman E. Breslow",
title = "Discussion on Professor {C}ox's Paper",
journal = "Journal of the Royal Statistical Society: Series B (Methodological)", 
volume =  "34",
number = "2", 
year = "1972",
pages = "216--217"}

@Book{Kalbfleisch, 
author = "John D. Kalbfleisch and Ross L. Prentice", 
title =  "The Statistical Analysis of Failure Time Data", 
edition   = "Second",
publisher = "Wiley-Interscience", 
year =  "2002" 
}

@Book{Cox, 
author =       "D. R. Cox and David Oakes", 
title =        "Analysis of Survival Data", 
publisher =   "Chapman and Hall/CRC", 
year =  "1984" 
}

@article{Ping, 
author =  "Ping Wang and Yan Li and Chandan K. Reddy",
title = "Machine Learning for Survival Analysis: A Survey",
journal = "ACM Computing Surveys",
volume =  "51",
number = "6",
year = "2019",
pages = "1--36"}

@article{Dempster, 
author =  "Arthur P. Dempster and Nan M. Laird and Donald B. Rubin",
title = "Maximum Likelihood from Incomplete Data Via the {EM} Algorithm",
journal = "Journal of the Royal Statistical Society: Series B (Methodological)", 
volume =  "39",
number = "1", 
year = "1977",
pages = "1--22"}

@article{Erika, 
author =  "Erika Graf and Claudia Schmoor and Willi Sauerbrei and Martin Schumacher", 
title = "Assessment and comparison of prognostic classification schemes for survival data",
journal = "Statistics in Medicine",
volume =  "18",
year = "1999",
pages = "2529--2545"}

@article{Farrington, 
author =  "C. P. Farrington and N. J. Andrews and A. D. Beale and M. A. Catchpole", 
title = "A Statistical Algorithm for the Early Detection of Outbreaks of Infectious Disease", 
journal = "Journal of the Royal Statistical Society. Series A (Statistics in Society)",
volume =  "159",
number = "3",
year = "1996",
pages = "547--563"}

@article{Angela, 
author =  "Angela Noufaily and Doyo G. Enki and Paddy Farrington and Paul Garthwaite and Nick Andrews and André Charlett", 
title = "An improved algorithm for outbreak detection in multiple surveillance systems", 
journal = "Statistics in Medicine",
volume =  "32",
number = "7",
year = "2013",
pages = "1206--1222"}

@article{Kodell, 
author =  "R. L. Kodell and C. J. Nelson", 
title = "An Illness-Death Model for the Study of the Carcinogenic Process Using Survival/Sacrifice Data", 
journal = "Biometrics",
volume =  "36",
number = "2",
year = "1980",
pages = "267--277"}

@article{Bruce, 
author =  "Bruce W. Turnbull and Toby J. Mitchell", 
title = "Nonparametric Estimation of the Distribution of Time to Onset for Specific Diseases in Survival/Sacrifice Experiments", 
journal = "Biometrics",
volume =  "40",
number = "1",
year = "1984",
pages = "41--50"}

@article{Laan, 
author =  "Mark J. Van Der Laan and Nicholas P. Jewell and Derick R. Peterson", 
title = "Efficient Estimation of the Lifetime and Disease Onset Distribution", 
journal = "Biometrika",
volume =  "84",
number = "3",
year = "1997",
pages = "539--554"}

@article{Tibshirani, 
author =  "Trevor Hastie and Robert Tibshirani", 
title = "Generalized Additive Models", 
journal = "Statistical Science",
volume =  "1",
number = "3",
year = "1986",
pages = "297--310"}

@article{LeBlanc, 
author =  "Michael LeBlanc and John Crowley", 
title = "Relative Risk Trees for Censored Survival Data", 
journal = "Biometrics",
volume =  "48",
number = "2",
year = "1992",
pages = "411--425"}

@article{Jared, 
author =  "Jared L. Katzman and Uri Shaham and Alexander Cloninger and Jonathan Bates and Tingting Jiang and Yuval Kluger", 
title = "DeepSurv: personalized treatment recommender system using a Cox proportional hazards deep neural network", 
journal = "BMC Medical Research Methodology",
volume =  "18",
number = "24",
year = "2018"}

@article{Faraggi, 
author =  "David Faraggi and Richard Simon", 
title = "A neural network model for survival data", 
journal = "Statistics in Medicine",
volume =  "14",
number = "1",
year = "1995",
pages = "73--82"}

@article{Scheel, 
author =  "Håvard Kvamme and Ørnulf Borgan and Ida Scheel", 
title = "Time-to-Event Prediction with Neural Networks and {C}ox Regression", 
journal = "Journal of Machine Learning Research",
volume =  "20",
number = "129",
year = "2019",
pages = "1--30"}

@inproceedings{Zhong, 
author = "Qixian Zhong and Jonas Mueller and Jane-Ling Wang", 
title = "Deep Extended Hazard Models for Survival Analysis", 
booktitle = "Proceedings of the 35th International Conference on Neural Information Processing Systems",
year = "2021",
pages = "15111--15124" 
}

@inproceedings{Nagpal, 
author = "Chirag Nagpal and Steve Yadlowsky and Negar Rostamzadeh and Katherine Heller", 
title = "Deep {C}ox Mixtures for Survival Regression", 
booktitle = "Proceedings of the 6th Machine Learning for Healthcare Conference",
year = "2021",
pages = "674--708" 
}

@inproceedings{Jiachang, 
author = "Jiachang Liu and Rui Zhang and Cynthia Rudin", 
title = "FastSurvival: Hidden Computational Blessings in Training {C}ox Proportional Hazards Models", 
booktitle = "Proceedings of the 38th International Conference on Neural Information Processing Systems",
year = "2024",
pages = "87712--87765" 
}

@inproceedings{Csaba, 
author =   "Pooria Joulani and Andras Gyorgy and Csaba Szepesvari", 
title = "Online Learning under Delayed Feedback", 
booktitle = "Proceedings of the 30th International Conference on Machine Learning",
year = "2013", 
pages = "1453--1461" 
}

@inproceedings{Kent, 
author =   "Kent Quanrud and Daniel Khashabi", 
title = "Online Learning with Adversarial Delays", 
booktitle = "Proceedings of the 29th International Conference on Neural Information Processing Systems",
year = "2015", 
pages = "1270--1278" 
}

@inproceedings{Mann, 
author =   "Timothy A. Mann and Sven Gowal and Andras Gyorgy and Huiyi Hu and Ray Jiang and Balaji Lakshminarayanan and Prav Srinivasan", 
title = "Proxies with Applications to Recommender Systems", 
booktitle = "Proceedings of the 36th International Conference on Machine Learning",
year = "2019", 
pages = "4324--4332" 
}

@inproceedings{Ilai, 
author =   "Ilai Bistritz and Zhengyuan Zhou and Xi Chen and Nicholas Bambos and Jose Blanchet", 
title = "Online {EXP3} learning in adversarial bandits with delayed feedback", 
booktitle = "Proceedings of the 33rd International Conference on Neural Information Processing Systems",
year = "2019", 
pages = "11349--11358" 
}

@inproceedings{Brueckne, 
author =   "Claire Vernade and Alexandra Carpentier and Tor Lattimore and Giovanni Zappella and Beyza Ermis and Michael Brueckne", 
title = "Linear bandits with stochastic delayed feedback", 
booktitle = "Proceedings of the 37th International Conference on Machine Learning",
year = "2020", 
pages = "9712--9721" 
}

@inproceedings{Seldin, 
author =   "Julian Zimmert and Saeed Masoudian and Yevgeny Seldin", 
title = "A best-of-both-worlds algorithm for bandits with delayed feedback with robustness to excessive delays", 
booktitle = "Proceedings of the 38th International Conference on Neural Information Processing Systems",
year = "2024", 
pages = "141071--141102" 
}

@inproceedings{Chandan, 
author =   "Yan Li and Lu Wang and Jie Wang and Jieping Ye and Chandan K. Reddy", 
title = "Transfer Learning for Survival Analysis via Efficient {L2,1}-Norm Regularized Cox Regression", 
booktitle = "Proceedings of the 16th IEEE International Conference on Data Mining",
year = "2016", 
pages = "231--240" 
}

@inproceedings{Bellot, 
author =   "Alexis Bellot and Mihaela van der Schaar", 
title = "Boosting Transfer Learning with Survival Data from Heterogeneous Domains", 
booktitle = "Proceedings of the 32nd International Conference on Artificial Intelligence and Statistics",
year = "2019", 
pages = "57--65" 
}

@inproceedings{Carolin, 
author =   "Ammar Shaker and Carolin Lawrence", 
title = "Multi-Source Survival Domain Adaptation", 
booktitle = "Proceedings of the 37th AAAI Conference on Artificial Intelligence",
year = "2023", 
pages = "9752--9762" 
}

@inproceedings{Ethan, 
author =   "Ethan Steinberg and Jason A. Fries and Yizhe Xu and Nigam H. Shah", 
title = "{MOTOR}: A Time-to-Event Foundation Model For Structured Medical Records", 
booktitle = "International Conference on Learning Representations",
year = "2024" 
}

@article{Julie, 
author =  "Pan Zhao and Julie Josse and Shu Yang", 
title = "Efficient and robust transfer learning of optimal individualized treatment regimes with right-censored survival data", 
journal = "Journal of Machine Learning Research",
volume =  "26",
number = "48",
year = "2025",
pages = "1--54"}

@article{Buchardt,  
author =  "Kristian Buchardt and Christian Furrer and Oliver Sandqvist", 
title = "Estimation for multistate models subject to reporting delays and incomplete event adjudication with application to disability insurance", 
journal = "Annals of Applied Statistics", 
volume =  "19",
number = "4", 
year = "2025",
pages = "2782--2802"}

\newpage  
\appendix

\section{Notations}
\label{Description}
\begin{table}[h]
  \centering
  \begin{center}
  \begin{tabular}{c|l} 
    \toprule
    Notation  & \multicolumn{1}{c}{Description}  \\ 	
    \midrule 
    $\mathcal{D}$ & $\mathcal{D} \equiv (\bm{x}_i, y_i, v_i)_{i=1}^n$  \\ 
    $\mathcal{D}_1$ & $\mathcal{D}_1 \equiv (\bm{x}_i, z_i, w_i)_{i=1}^n$  \\  
    $\mathcal{D}_2$ & $\mathcal{D}_2 \equiv (\bm{x}_i, y_i - z_i, v_i)_{i=1}^n$  \\ 
    $\mathcal{P}_1$ & the combined pseudo-complete dataset corresponding to the earlier process \\ 
    $\mathcal{P}_2$ & the combined pseudo-complete dataset corresponding to the later process \\ 
    $n$ & the number of observations \\ 
    $m$ & the number of observations with $v_i = 0$ \\ 
    $s$ & the number of pseudo-complete datasets  \\ 
    $Y$ & $\min \{T, C\}$ \\ 
    $\bar{Y}$ & $\min\{Y, \tau\}$, where $\tau$ denotes the administrative censoring time \\  
    $V$ & $\mathbb{I}(T \leq C)$ \\  
    $\bar{V}$ & $\mathbb{I}(T \leq C, T \leq \tau)$ \\  
    $Z$ & $\min \{T_1 , C\}$ \\ 
    $\bar{Z}$ & $\min\{Z, \tau\}$ \\ 
    $W$ & $\mathbb{I}(T_1 \leq C)$ \\  
    $\bar{W}$ & $\mathbb{I}(T_1 \leq C, T_1 \leq \tau)$ \\   
    $\bm{x}_i$ & $\bm{x}_i \in \mathcal{X} \subset \mathbb{R}^d$ represents the individual covariates \\ 
    $y_i$ & the time from the start of the earlier process to the later event or censoring \\ 
    $v_i$ & the later event indicator \\ 
    $z_i$ & the time until the earlier event or censoring \\ 
    $w_i$ & the earlier event indicator \\ 
    $\mathcal{L}(\bm{\theta})$ & $\log \prod_{i = 1}^n f(y_i \mid \bm{x}_i)^{v_i} S(y_i \mid \bm{x}_i)^{1 - v_i}$ \\ 
    $\mathcal{L}_c(\bm{\theta}_c)$ & $\log \prod_{i = 1}^n S_c(y_i \mid \bm{x}_i)^{v_i} f_c(y_i \mid \bm{x}_i)^{1 - v_i}$ \\   
    $\mathcal{L}_c^\dagger(\bm{\theta}_c)$ & $\log \prod_{i = 1}^n S_c(y_i \mid \bm{x}_i)^{v_i^\dagger} f_c(y_i \mid \bm{x}_i)^{1 - v_i^\dagger}$ with $v_i^\dagger \equiv \mathbb{I}(y_i \geq \tau (1 - v_i))$ \\   
    $\mathcal{E}(\bm{\theta}_1, \bm{\theta}_2)$ & $\log \prod_{i = 1}^n S_{\scriptscriptstyle\blacksquare}(y_i \mid \bm{x}_i)^{1 - v_i} f_{\scriptscriptstyle\square}(z_i, y_i \mid \bm{x}_i)^{v_i}$ \\ 
    $f(t \mid \bm{x})$ & the conditional density function of $T$ at $t$ given $X = \bm{x}$ \\  
    $f_c(t \mid \bm{x})$ & the conditional density function of $C$ at $t$ given $X = \bm{x}$ \\  
    $f_1(t \mid \bm{x})$ & the conditional density function of $T_1$ at $t$ given $X = \bm{x}$ \\  
    $f_2(t \mid \bm{x})$ & the conditional density function of $T_2$ at $t$ given $X = \bm{x}$ \\  
    $f_{\scriptscriptstyle\square}(z, y \mid \bm{x})$ & $f_1(z \mid \bm{x}) f_2(y - z \mid \bm{x})$ \\ 
    $S(t \mid \bm{x})$ & $\int_t^{\infty} f(s \mid \bm{x}) ds$ \\  
    $S_c(t \mid \bm{x})$ & $\int_t^{\infty} f_c(s \mid \bm{x}) ds$ \\  
    $S_1(t \mid \bm{x})$ & $\int_t^{\infty} f_1(s \mid \bm{x}) ds$ \\  
    $S_2(t \mid \bm{x})$ & $\int_t^{\infty} f_2(s \mid \bm{x}) ds$ \\ 
    $S_{\scriptscriptstyle\blacksquare}(y \mid \bm{x})$ & $S_1(y \mid \bm{x}) + \int_0^y f_1(t \mid \bm{x}) S_2(y - t \mid \bm{x}) dt$ \\ 
    $\bm{\theta}$ & the real-valued vector that parameterizes $f$ \\ 
    $\bm{\theta}_c$ & the real-valued vector that parameterizes $f_c$ \\ 
    $\bm{\theta}_1$ & the real-valued vector that parameterizes $f_1$ \\ 
    $\bm{\theta}_2$ & the real-valued vector that parameterizes $f_2$ \\ 
    \bottomrule 
  \end{tabular}
  \end{center}
\end{table}

\section{Related Work} 
\subsection{Proportional Hazards Model} 
\label{Cox} 
The Cox proportional hazards model \citep{Cox0} is widely used to represent time dependence and covariate effect, assuming a constant hazard ratio over time. 
The regression coefficients capturing covariate effect are estimated via partial likelihood \citep{Cox1}, after which the baseline hazard characterizing time dependence is estimated nonparametrically \citep{Breslow0}. 
Alternatively, the baseline hazard can be modeled parametrically using piecewise constant functions \citep{Holford}. 
A large body of work has focused on enhancing the flexibility of covariate effect, 
replacing the linear predictor with nonlinear structures such as additive splines \citep{Tibshirani}, decision trees \citep{LeBlanc}, and neural networks \citep{Faraggi,Jared}. 
More recently, extensions beyond the proportional hazards assumption have been studied in the machine learning community \citep{Scheel,Zhong,Nagpal}. 
Additionally, efficient and stable optimization for Cox regression has become increasingly important in large-scale and high-dimensional settings \citep{Jiachang}.

\subsection{Delay Structures in Related Fields} 
\label{delay} 
Delay structures arising from operational constraints have been studied in fields such as nowcasting \citep{Matthias, Stoner1, Wood, Stoner2}, outbreak detection \citep{Farrington, Angela}, and online learning \citep{Csaba, Kent, Mann, Ilai, Brueckne, Seldin}.  
Nowcasting and outbreak detection typically correct reporting lags in aggregated case counts for real-time estimation or surveillance.  
In online learning, delayed feedback similarly complicates sequential decision-making, since the learner must update its model or policy before outcomes from earlier actions are revealed.

\section{Preliminaries} 
\label{sup_preliminaries}

\subsection{Standard Survival Model without Administrative Censoring} 
\label{mix} 
For $y \in [0, \infty)$, $v = 0$, and $\bm{x} \in \mathcal{X}$, the conditional joint distribution of $(Y, V)$ is given by  
\begin{align} 
  \frac{\partial}{\partial y} \Pr(Y \leq y, V = 0 \mid X = \bm{x})  
  &= \frac{\partial}{\partial y} \Pr(C \leq y, C < T \mid X = \bm{x})& \nonumber  \\ 
  &= \frac{\partial}{\partial y} \int_0^y dc \int_c^{\infty} dt f(t \mid \bm{x}) f_c(c \mid \bm{x})& \nonumber  \\  
  &= \int_y^{\infty} f(t \mid \bm{x}) f_c(y \mid \bm{x}) dt& \nonumber  \\  
  &= S(y \mid \bm{x}) f_c(y \mid \bm{x}).&   
\end{align} 
For $y \in [0, \infty)$, $v = 1$, and $\bm{x} \in \mathcal{X}$, the conditional joint distribution of $(Y, V)$ is given by 
\begin{align} 
  \frac{\partial}{\partial y} \Pr(Y \leq y, V = 1 \mid X = \bm{x})   
  &= \frac{\partial}{\partial y} \Pr(T \leq y, T \leq C \mid X = \bm{x})& \nonumber  \\ 
  &= \frac{\partial}{\partial y} \int_0^y dt \int_t^{\infty} dc f(t \mid \bm{x}) f_c(c \mid \bm{x})& \nonumber  \\  
  &= \int_y^{\infty} f(y \mid \bm{x}) f_c(c \mid \bm{x}) dc& \nonumber  \\  
  &= f(y \mid \bm{x}) S_c(y \mid \bm{x}).&   
\end{align} 
Consequently, the conditional joint distribution of $(Y, V)$ is given by $g(y, v \mid \bm{x})$.  
Then the log likelihood without administrative censoring is given by 
\begin{align} 
  \log \prod_{i = 1}^n g(y_i, v_i \mid \bm{x}_i) 
  &= \log \prod_{i = 1}^n \bigl(S(y_i \mid \bm{x}_i) f_c(y_i \mid \bm{x}_i)\bigr)^{1 - v_i} \bigl(f(y_i \mid \bm{x}_i) S_c(y_i \mid \bm{x}_i)\bigr)^{v_i}&  \nonumber \\      
  &= \log \prod_{i = 1}^n S(y_i \mid \bm{x}_i)^{1 - v_i} f(y_i \mid \bm{x}_i)^{v_i} f_c(y_i \mid \bm{x}_i)^{1 - v_i} S_c(y_i \mid \bm{x}_i)^{v_i}&  \nonumber \\         
  &= \mathcal{L}(\bm{\theta}) + \mathcal{L}_c(\bm{\theta}_c).& 
\end{align} 
From the definition of $g$, for any $y \in [0, \infty)$ and $\bm{x} \in \mathcal{X}$, we have 
\begin{align} 
&S(y \mid \bm{x}) S_c(y \mid \bm{x}) 
= 1 - \int_0^y \bigl(f(t \mid \bm{x}) S_c(t \mid \bm{x}) + S(t \mid \bm{x}) f_c(t \mid \bm{x})\bigr) dt  
= 1 - \sum_{v \in \{0, 1\}} \int_0^y g(t, v \mid \bm{x}) dt, \\ 
&S(y \mid \bm{x}) 
= \exp\bigl(\log S(y \mid \bm{x})\bigr)  
= \exp\Bigl(- \int_0^y \frac{f(t \mid \bm{x})}{S(t \mid \bm{x})} dt\Bigr) 
= \exp\Bigl(- \int_0^y \frac{g(t, 1 \mid \bm{x})}{S(t \mid \bm{x}) S_c(t \mid \bm{x})} dt\Bigr), \\ 
&S_c(y \mid \bm{x}) 
= \exp\bigl(\log S_c(y \mid \bm{x})\bigr) 
= \exp\Bigl(- \int_0^y \frac{f_c(t \mid \bm{x})}{S_c(t \mid \bm{x})} dt\Bigr) 
= \exp\Bigl(- \int_0^y \frac{g(t, 0 \mid \bm{x})}{S(t \mid \bm{x}) S_c(t \mid \bm{x})} dt\Bigr). 
\end{align} 
Therefore, if $(f, f_c)$ is identifiable with respect to $(\bm{\theta}, \bm{\theta}_c)$, then $g$ is also identifiable with respect to $(\bm{\theta}, \bm{\theta}_c)$. 
By the uniform law of large numbers and Jensen's inequality, the following holds:  
\begin{align} 
  \frac{1}{n} \bigl(\mathcal{L}(\bm{\theta}) + \mathcal{L}_c(\bm{\theta}_c)\bigr)   
  &= \frac{1}{n} \sum_{i = 1}^n \log g(y_i, v_i \mid \bm{x}_i; \bm{\theta}, \bm{\theta}_c)&  \nonumber \\ 
  &\xrightarrow{p} \int_{\mathcal{X}} dP_X(\bm{x}) \int_0^{\infty} dy \sum_{v \in \{0, 1\}} g(y, v \mid \bm{x}; \bm{\theta}^*, \bm{\theta}_c^*) \log g(y, v \mid \bm{x}; \bm{\theta}, \bm{\theta}_c)&  \nonumber \\  
  &= \int_{\mathcal{X}} dP_X(\bm{x}) \int_0^{\infty} dy \sum_{v \in \{0, 1\}} g(y, v \mid \bm{x}; \bm{\theta}^*, \bm{\theta}_c^*) \log \frac{g(y, v \mid \bm{x}; \bm{\theta}, \bm{\theta}_c)}{g(y, v \mid \bm{x}; \bm{\theta}^*, \bm{\theta}_c^*)}&  \nonumber \\  
  &~~~~+ \int_{\mathcal{X}} dP_X(\bm{x}) \int_0^{\infty} dy \sum_{v \in \{0, 1\}} g(y, v \mid \bm{x}; \bm{\theta}^*, \bm{\theta}_c^*) \log g(y, v \mid \bm{x}; \bm{\theta}^*, \bm{\theta}_c^*)&  \nonumber \\ 
  &\leq \int_{\mathcal{X}} dP_X(\bm{x}) \int_0^{\infty} dy \sum_{v \in \{0, 1\}} g(y, v \mid \bm{x}; \bm{\theta}^*, \bm{\theta}_c^*) \log g(y, v \mid \bm{x}; \bm{\theta}^*, \bm{\theta}_c^*),&  
\end{align} 
where $\bm{\theta}^*$ and $\bm{\theta}_c^*$ are the true values of $\bm{\theta}$ and $\bm{\theta}_c$, respectively. 
Consider the equality condition in Jensen's inequality. 
If $(f, f_c)$ is identifiable with respect to $(\bm{\theta}, \bm{\theta}_c)$, then the estimator of $(\bm{\theta}, \bm{\theta}_c)$ obtained by maximizing $\mathcal{L}(\bm{\theta}) + \mathcal{L}_c(\bm{\theta}_c)$ is asymptotically consistent. 
Similarly, the following holds:  
\begin{align} 
  \frac{1}{n} \mathcal{L}(\bm{\theta})  
  &= \frac{1}{n} \sum_{i = 1}^n \log f(y_i \mid \bm{x}_i; \bm{\theta})^{v_i} S(y_i \mid \bm{x}_i; \bm{\theta})^{1 - v_i}&  \nonumber \\ 
  &\xrightarrow{p} \int_{\mathcal{X}} dP_X(\bm{x}) \int_0^{\infty} dy \sum_{v \in \{0, 1\}} g(y, v \mid \bm{x}; \bm{\theta}^*, \bm{\theta}_c^*) \log f(y \mid \bm{x}; \bm{\theta})^v S(y \mid \bm{x}; \bm{\theta})^{1 - v}&  \nonumber \\  
  &= \int_{\mathcal{X}} dP_X(\bm{x}) \int_0^{\infty} dy \sum_{v \in \{0, 1\}} g(y, v \mid \bm{x}; \bm{\theta}^*, \bm{\theta}_c^*) \log \frac{f(y \mid \bm{x}; \bm{\theta})^v S(y \mid \bm{x}; \bm{\theta})^{1 - v}}{f(y \mid \bm{x}; \bm{\theta}^*)^v S(y \mid \bm{x}; \bm{\theta}^*)^{1 - v}}&  \nonumber \\  
  &~~~~+ \int_{\mathcal{X}} dP_X(\bm{x}) \int_0^{\infty} dy \sum_{v \in \{0, 1\}} g(y, v \mid \bm{x}; \bm{\theta}^*, \bm{\theta}_c^*) \log f(y \mid \bm{x}; \bm{\theta}^*)^v S(y \mid \bm{x}; \bm{\theta}^*)^{1 - v}&  \nonumber \\  
  &\leq \int_{\mathcal{X}} dP_X(\bm{x}) \int_0^{\infty} dy \sum_{v \in \{0, 1\}} g(y, v \mid \bm{x}; \bm{\theta}^*, \bm{\theta}_c^*) \log f(y \mid \bm{x}; \bm{\theta}^*)^v S(y \mid \bm{x}; \bm{\theta}^*)^{1 - v}.&  
\end{align} 
Therefore, if $f$ is identifiable with respect to $\bm{\theta}$, then the estimator of $\bm{\theta}$ obtained by maximizing $\mathcal{L}(\bm{\theta})$ is asymptotically consistent.

\subsection{Standard Survival Model with Administrative Censoring} 
\label{mix2} 
For $y = \tau$, $v = 0$, and $\bm{x} \in \mathcal{X}$, the conditional joint distribution of $(\bar{Y}, \bar{V})$ is given by  
\begin{align} 
\Pr(\bar{Y} = \tau, \bar{V} = 0 \mid X = \bm{x})   
&= \Pr(\tau \leq C, \tau < T \mid X = \bm{x}) \nonumber  \\ 
&= \int_{\tau}^{\infty} dc \int_{\tau}^{\infty} dt f(t \mid \bm{x}) f_c(c \mid \bm{x}) \nonumber  \\  
&= S(\tau \mid \bm{x}) S_c(\tau \mid \bm{x}).   
\end{align} 
For $y \in [0, \tau)$, $v = 0$, and $\bm{x} \in \mathcal{X}$, the conditional joint distribution of $(\bar{Y}, \bar{V})$ is given by  
\begin{align} 
\frac{\partial}{\partial y} \Pr(\bar{Y} \leq y, \bar{V} = 0 \mid X = \bm{x}) 
= \frac{\partial}{\partial y} \Pr(Y \leq y, V = v \mid X = \bm{x})  
= S(y \mid \bm{x}) f_c(y \mid \bm{x}).   
\end{align} 
For $y \in [0, \tau)$, $v = 1$, and $\bm{x} \in \mathcal{X}$, the conditional joint distribution of $(\bar{Y}, \bar{V})$ is given by  
\begin{align} 
\frac{\partial}{\partial y} \Pr(\bar{Y} \leq y, \bar{V} = 1 \mid X = \bm{x}) 
= \frac{\partial}{\partial y} \Pr(Y \leq y, V = v \mid X = \bm{x})  
= f(y \mid \bm{x}) S_c(y \mid \bm{x}). 
\end{align} 
Consequently, the conditional joint distribution of $(\bar{Y}, \bar{V})$ is given by $\bar{g}(y, v \mid \bm{x})$.  
Then the log likelihood with administrative censoring is given by 
\begin{align} 
&\log \prod_{i = 1}^n \bar{g}(y_i, v_i \mid \bm{x}_i)&  \nonumber \\   
&= \log \prod_{i = 1}^n \bigl(S(y_i \mid \bm{x}_i) S_c(y_i \mid \bm{x}_i)\bigr)^{(1 - v_i) \mathbb{I}(y_i = \tau)} \bigl(S(y_i \mid \bm{x}_i) f_c(y_i \mid \bm{x}_i)\bigr)^{(1 - v_i) \mathbb{I}(y_i < \tau)} \bigl(f(y_i \mid \bm{x}_i) S_c(y_i \mid \bm{x}_i)\bigr)^{v_i}&  \nonumber \\   
&= \log \prod_{i = 1}^n S(y_i \mid \bm{x}_i)^{1 - v_i} f(y_i \mid \bm{x}_i)^{v_i} f_c(y_i \mid \bm{x}_i)^{1 - v_i^\dagger} S_c(y_i \mid \bm{x}_i)^{v_i^\dagger}&  \nonumber \\         
&= \mathcal{L}(\bm{\theta}) + \mathcal{L}_c^\dagger(\bm{\theta}_c).& 
\end{align} 
Consider the setting obtained by replacing the time interval $[0, \infty)$ with $[0, \tau]$ and the observed-data distribution characterized by $g$ with that characterized by $\bar{g}$. 
If $(f, f_c)$ over the time interval $[0, \tau]$ is identifiable with respect to $(\bm{\theta}, \bm{\theta}_c)$, then the estimator of $(\bm{\theta}, \bm{\theta}_c)$ obtained by maximizing $\mathcal{L}(\bm{\theta}) + \mathcal{L}_c^\dagger(\bm{\theta}_c)$ is asymptotically consistent. 
If $f$ over the time interval $[0, \tau]$ is identifiable with respect to $\bm{\theta}$, the estimator of $\bm{\theta}$ obtained by maximizing $\mathcal{L}(\bm{\theta})$ is asymptotically consistent.

\section{Proofs} 
\label{proofs} 

\subsection{Proof of \Cref{mixture}} 
For $0 \leq y$, $v = 0$, and $\bm{x} \in \mathcal{X}$, the conditional joint distribution of $(Y, V)$ is given by 
\begin{align} 
  &\frac{\partial}{\partial y} \Pr(Y \leq y, V = 0 \mid X = \bm{x})  \nonumber  \\ 
  &= \frac{\partial}{\partial y} \Pr(Y \leq y, W = 0, V = 0 \mid X = \bm{x}) + \frac{\partial}{\partial y} \Pr(Y \leq y, W = 1, V = 0 \mid X = \bm{x}) \nonumber  \\ 
  &= \frac{\partial}{\partial y} \Pr(C \leq y, C < T_1 \mid X = \bm{x}) + \frac{\partial}{\partial y} \Pr(C \leq y, T_1 \leq C, C < T_1 + T_2 \mid X = \bm{x}) \nonumber  \\ 
  &= \frac{\partial}{\partial y} \int_0^y dc \int_c^{\infty} dt_1 f_1(t_1 \mid \bm{x}) f_c(c \mid \bm{x}) + \frac{\partial}{\partial y} \int_0^y dc \int_0^c dt_1 \int_{c - t_1}^{\infty} dt_2 f_1(t_1 \mid \bm{x}) f_2(t_2 \mid \bm{x}) f_c(c \mid \bm{x}) \nonumber  \\  
  &= \int_y^{\infty} dt_1 f_1(t_1 \mid \bm{x}) f_c(y \mid \bm{x}) + \int_0^y dt_1 \int_{y - t_1}^{\infty} dt_2 f_1(t_1 \mid \bm{x}) f_2(t_2 \mid \bm{x}) f_c(y \mid \bm{x}) \nonumber  \\  
  &= S_1(y \mid \bm{x}) f_c(y \mid \bm{x}) + \int_0^y f_1(t_1 \mid \bm{x}) S_2(y - t_1 \mid \bm{x}) f_c(y \mid \bm{x}) dt_1 \nonumber  \\   
  &= S_{\scriptscriptstyle\blacksquare}(y \mid \bm{x}) f_c(y \mid \bm{x}).   
\end{align} 
For $0 \leq z \leq y$, $w = v = 1$, and $\bm{x} \in \mathcal{X}$, the conditional joint distribution of $(Z, Y, W, V)$ is given by 
\begin{align} 
  &\frac{\partial^2}{\partial y \partial z} \Pr(Z \leq z, Y \leq y, W = 1, V = 1 \mid X = \bm{x})  \nonumber  \\ 
  &= \frac{\partial^2}{\partial y \partial z} \Pr(T_1 \leq z, T_1 + T_2 \leq y, T_1 \leq C, T_1 + T_2 \leq C \mid X = \bm{x}) \nonumber  \\ 
  &= \frac{\partial^2}{\partial y \partial z} \int_0^z dt_1 \int_0^{y - t_1} dt_2 \int_{t_1 + t_2}^{\infty} dc f_1(t_1 \mid \bm{x}) f_2(t_2 \mid \bm{x}) f_c(c \mid \bm{x}) \nonumber  \\  
  &= \frac{\partial^2}{\partial y \partial z} \int_0^z dt_1 \int_0^{y - t_1} dt_2 f_1(t_1 \mid \bm{x}) f_2(t_2 \mid \bm{x}) S_c(t_1 + t_2 \mid \bm{x}) \nonumber  \\  
  &= \frac{\partial}{\partial y} \int_0^{y - z} f_1(z \mid \bm{x}) f_2(t_2 \mid \bm{x}) S_c(z + t_2 \mid \bm{x}) dt_2 \nonumber  \\  
  &= f_1(z \mid \bm{x}) f_2(y - z \mid \bm{x}) S_c(y \mid \bm{x}) \nonumber  \\  
  &= f_{\scriptscriptstyle\square}(z, y \mid \bm{x}) S_c(y \mid \bm{x}). 
\end{align}

\subsection{Proof of \Cref{loglikelihood_without}} 
Applying \Cref{mixture}, the log likelihood without administrative censoring is given by 
\begin{align} 
\log \prod_{i = 1}^n g_{\scriptscriptstyle\blacksquare}(y_i \mid \bm{x}_i)^{1 - v_i} g_{\scriptscriptstyle\square}(z_i, y_i \mid \bm{x}_i)^{v_i} 
&= \log \prod_{i = 1}^n \bigl(S_{\scriptscriptstyle\blacksquare}(y_i \mid \bm{x}_i) f_c(y_i \mid \bm{x}_i)\bigr)^{1 - v_i} \bigl(f_{\scriptscriptstyle\square}(z_i, y_i \mid \bm{x}_i) S_c(y_i \mid \bm{x}_i)\bigr)^{v_i}  \nonumber \\      
&= \log \prod_{i = 1}^n S_{\scriptscriptstyle\blacksquare}(y_i \mid \bm{x}_i)^{1 - v_i} f_{\scriptscriptstyle\square}(z_i, y_i \mid \bm{x}_i)^{v_i} f_c(y_i \mid \bm{x}_i)^{1 - v_i} S_c(y_i \mid \bm{x}_i)^{v_i}  \nonumber \\         
&= \mathcal{E}(\bm{\theta}_1, \bm{\theta}_2) + \mathcal{L}_c(\bm{\theta}_c). 
\end{align}

\subsection{Proof of \Cref{ind2}} 
\label{ind2_proof} 
For any $t_1, t_2 \in [0, \infty)$ and $\bm{x} \in \mathcal{X}$, the following holds:  
\begin{align}
f_1(t_1 \mid \bm{x}) = \int_0^\infty f_{\scriptscriptstyle\square}(t_1, t_1 + t \mid \bm{x}) dt, \qquad f_2(t_2 \mid \bm{x}) = \int_0^\infty f_{\scriptscriptstyle\square}(t, t + t_2 \mid \bm{x}) dt. 
\end{align} 
Therefore, under \Cref{f_identifiability}, 
if $f_{\scriptscriptstyle\square}(z, y \mid \bm{x}; \bm{\theta}_1, \bm{\theta}_2) = f_{\scriptscriptstyle\square}(z, y \mid \bm{x}; \bm{\theta}_1^\prime, \bm{\theta}_2^\prime)$ for Lebesgue-almost every $0 \leq z \leq y$ and $P_X$-almost every $\bm{x} \in \mathcal{X}$, 
then $(\bm{\theta}_1, \bm{\theta}_2) = (\bm{\theta}_1^\prime, \bm{\theta}_2^\prime)$. 
From \Cref{mixture}, the uniform law of large numbers, and Jensen's inequality, we have 
\begin{align} 
  \frac{1}{n} \mathcal{E}(\bm{\theta}_1, \bm{\theta}_2)   
  &= \frac{1}{n} \sum_{i = 1}^n \bigl((1 - v_i) \log S_{\scriptscriptstyle\blacksquare}(y_i \mid \bm{x}_i; \bm{\theta}_1, \bm{\theta}_2) + v_i \log f_{\scriptscriptstyle\square}(z_i, y_i \mid \bm{x}_i; \bm{\theta}_1, \bm{\theta}_2)\bigr)&  \nonumber \\ 
  &\xrightarrow{p} \int_{\mathcal{X}} dP_X(\bm{x}) \int_0^{\infty} dy g_{\scriptscriptstyle\blacksquare}(y \mid \bm{x}; \bm{\theta}_1^*, \bm{\theta}_2^*, \bm{\theta}_c^*) \log S_{\scriptscriptstyle\blacksquare}(y \mid \bm{x}; \bm{\theta}_1, \bm{\theta}_2)&  \nonumber \\  
  &~~~~+ \int_{\mathcal{X}} dP_X(\bm{x}) \int_0^{\infty} dy \int_0^y dz g_{\scriptscriptstyle\square}(z, y \mid \bm{x}; \bm{\theta}_1^*, \bm{\theta}_2^*, \bm{\theta}_c^*) \log f_{\scriptscriptstyle\square}(z, y \mid \bm{x}; \bm{\theta}_1, \bm{\theta}_2)&  \nonumber \\  
  &= \int_{\mathcal{X}} dP_X(\bm{x}) \int_0^{\infty} dy g_{\scriptscriptstyle\blacksquare}(y \mid \bm{x}; \bm{\theta}_1^*, \bm{\theta}_2^*, \bm{\theta}_c^*) \log \frac{S_{\scriptscriptstyle\blacksquare}(y \mid \bm{x}; \bm{\theta}_1, \bm{\theta}_2)}{S_{\scriptscriptstyle\blacksquare}(y \mid \bm{x}; \bm{\theta}_1^*, \bm{\theta}_2^*)}&  \nonumber \\  
  &~~~~+ \int_{\mathcal{X}} dP_X(\bm{x}) \int_0^{\infty} dy \int_0^y dz g_{\scriptscriptstyle\square}(z, y \mid \bm{x}; \bm{\theta}_1^*, \bm{\theta}_2^*, \bm{\theta}_c^*) \log \frac{f_{\scriptscriptstyle\square}(z, y \mid \bm{x}; \bm{\theta}_1, \bm{\theta}_2)}{f_{\scriptscriptstyle\square}(z, y \mid \bm{x}; \bm{\theta}_1^*, \bm{\theta}_2^*)}&  \nonumber \\  
  &~~~~+ \int_{\mathcal{X}} dP_X(\bm{x}) \int_0^{\infty} dy g_{\scriptscriptstyle\blacksquare}(y \mid \bm{x}; \bm{\theta}_1^*, \bm{\theta}_2^*, \bm{\theta}_c^*) \log S_{\scriptscriptstyle\blacksquare}(y \mid \bm{x}; \bm{\theta}_1^*, \bm{\theta}_2^*)&  \nonumber \\  
  &~~~~+ \int_{\mathcal{X}} dP_X(\bm{x}) \int_0^{\infty} dy \int_0^y dz g_{\scriptscriptstyle\square}(z, y \mid \bm{x}; \bm{\theta}_1^*, \bm{\theta}_2^*, \bm{\theta}_c^*) \log f_{\scriptscriptstyle\square}(z, y \mid \bm{x}; \bm{\theta}_1^*, \bm{\theta}_2^*)&  \nonumber \\ 
  &\leq \int_{\mathcal{X}} dP_X(\bm{x}) \int_0^{\infty} dy g_{\scriptscriptstyle\blacksquare}(y \mid \bm{x}; \bm{\theta}_1^*, \bm{\theta}_2^*, \bm{\theta}_c^*) \log S_{\scriptscriptstyle\blacksquare}(y \mid \bm{x}; \bm{\theta}_1^*, \bm{\theta}_2^*)&  \nonumber \\  
  &~~~~+ \int_{\mathcal{X}} dP_X(\bm{x}) \int_0^{\infty} dy \int_0^y dz g_{\scriptscriptstyle\square}(z, y \mid \bm{x}; \bm{\theta}_1^*, \bm{\theta}_2^*, \bm{\theta}_c^*) \log f_{\scriptscriptstyle\square}(z, y \mid \bm{x}; \bm{\theta}_1^*, \bm{\theta}_2^*).&  
\end{align} 
Consequently, from the equality condition in Jensen's inequality, the estimator of $(\bm{\theta}_1, \bm{\theta}_2)$ obtained by maximizing $\mathcal{E}(\bm{\theta}_1, \bm{\theta}_2)$ is asymptotically consistent.

\subsection{Proof of \Cref{ind_all}} 
By the definitions of $g_{\scriptscriptstyle\blacksquare}$ and $g_{\scriptscriptstyle\square}$, for $0 \leq z \leq y$ and $\bm{x} \in \mathcal{X}$, we have  
\begin{align}
S_{\scriptscriptstyle\blacksquare}(y \mid \bm{x}) S_c(y \mid \bm{x}) 
&= 1 - \int_0^y \Bigl(S_c(t \mid \bm{x}) \int_0^t f_1(u \mid \bm{x}) f_2(t - u \mid \bm{x}) du + S_{\scriptscriptstyle\blacksquare}(t \mid \bm{x}) f_c(t \mid \bm{x})\Bigr) dt  \nonumber  \\ 
&= 1 - \int_0^y \Bigl(\int_0^t g_{\scriptscriptstyle\square}(u, t \mid \bm{x}) du + g_{\scriptscriptstyle\blacksquare}(t \mid \bm{x})\Bigr) dt, \\ 
S_c(y \mid \bm{x}) 
&= \exp\bigl(\log S_c(y \mid \bm{x})\bigr) \nonumber  \\ 
&= \exp\Bigl(- \int_0^y \frac{f_c(t \mid \bm{x})}{S_c(t \mid \bm{x})} dt\Bigr) \nonumber  \\ 
&= \exp\Bigl(- \int_0^y \frac{g_{\scriptscriptstyle\blacksquare}(t \mid \bm{x})}{S_{\scriptscriptstyle\blacksquare}(t \mid \bm{x}) S_c(t \mid \bm{x})} dt\Bigr),   \\ 
f_{\scriptscriptstyle\square}(z, y \mid \bm{x}) 
&= \frac{g_{\scriptscriptstyle\square}(z, y \mid \bm{x})}{S_c(y \mid \bm{x})}. 
\end{align} 
Therefore, under \Cref{c_identifiability}, if $g_{\scriptscriptstyle\blacksquare}(y \mid \bm{x}; \bm{\theta}_1, \bm{\theta}_2, \bm{\theta}_c) = g_{\scriptscriptstyle\blacksquare}(y \mid \bm{x}; \bm{\theta}_1^\prime, \bm{\theta}_2^\prime, \bm{\theta}_c^\prime)$ and $g_{\scriptscriptstyle\square}(z, y \mid \bm{x}; \bm{\theta}_1, \bm{\theta}_2, \bm{\theta}_c) = g_{\scriptscriptstyle\square}(z, y \mid \bm{x}; \bm{\theta}_1^\prime, \bm{\theta}_2^\prime, \bm{\theta}_c^\prime)$ for Lebesgue-almost every $0 \leq z \leq y$ and $P_X$-almost every $\bm{x} \in \mathcal{X}$, 
then $(\bm{\theta}_1, \bm{\theta}_2, \bm{\theta}_c) = (\bm{\theta}_1^\prime, \bm{\theta}_2^\prime, \bm{\theta}_c^\prime)$. 
By \Cref{mixture}, the uniform law of large numbers, and Jensen's inequality, we have   
\begin{align} 
  \frac{1}{n} \bigl(\mathcal{E}(\bm{\theta}_1, \bm{\theta}_2) + \mathcal{L}_c(\bm{\theta}_c)\bigr) 
  &= \frac{1}{n} \sum_{i = 1}^n \bigl((1 - v_i) \log g_{\scriptscriptstyle\blacksquare}(y_i \mid \bm{x}_i; \bm{\theta}_1, \bm{\theta}_2, \bm{\theta}_c) + v_i \log g_{\scriptscriptstyle\square}(z_i, y_i \mid \bm{x}_i; \bm{\theta}_1, \bm{\theta}_2, \bm{\theta}_c)\bigr)&  \nonumber \\ 
  &\xrightarrow{p} \int_{\mathcal{X}} dP_X(\bm{x}) \int_0^{\infty} dy g_{\scriptscriptstyle\blacksquare}(y \mid \bm{x}; \bm{\theta}_1^*, \bm{\theta}_2^*, \bm{\theta}_c^*) \log g_{\scriptscriptstyle\blacksquare}(y \mid \bm{x}; \bm{\theta}_1, \bm{\theta}_2, \bm{\theta}_c)&  \nonumber \\  
  &~~~~+ \int_{\mathcal{X}} dP_X(\bm{x}) \int_0^{\infty} dy \int_0^y dz g_{\scriptscriptstyle\square}(z, y \mid \bm{x}; \bm{\theta}_1^*, \bm{\theta}_2^*, \bm{\theta}_c^*) \log g_{\scriptscriptstyle\square}(z, y \mid \bm{x}; \bm{\theta}_1, \bm{\theta}_2, \bm{\theta}_c)&  \nonumber \\  
  &= \int_{\mathcal{X}} dP_X(\bm{x}) \int_0^{\infty} dy g_{\scriptscriptstyle\blacksquare}(y \mid \bm{x}; \bm{\theta}_1^*, \bm{\theta}_2^*, \bm{\theta}_c^*) \log \frac{g_{\scriptscriptstyle\blacksquare}(y \mid \bm{x}; \bm{\theta}_1, \bm{\theta}_2, \bm{\theta}_c)}{g_{\scriptscriptstyle\blacksquare}(y \mid \bm{x}; \bm{\theta}_1^*, \bm{\theta}_2^*, \bm{\theta}_c^*)}&  \nonumber \\  
  &~~~~+ \int_{\mathcal{X}} dP_X(\bm{x}) \int_0^{\infty} dy \int_0^y dz g_{\scriptscriptstyle\square}(z, y \mid \bm{x}; \bm{\theta}_1^*, \bm{\theta}_2^*, \bm{\theta}_c^*) \log \frac{g_{\scriptscriptstyle\square}(z, y \mid \bm{x}; \bm{\theta}_1, \bm{\theta}_2, \bm{\theta}_c)}{g_{\scriptscriptstyle\square}(z, y \mid \bm{x}; \bm{\theta}_1^*, \bm{\theta}_2^*, \bm{\theta}_c^*)}&  \nonumber \\  
  &~~~~+ \int_{\mathcal{X}} dP_X(\bm{x}) \int_0^{\infty} dy g_{\scriptscriptstyle\blacksquare}(y \mid \bm{x}; \bm{\theta}_1^*, \bm{\theta}_2^*, \bm{\theta}_c^*) \log g_{\scriptscriptstyle\blacksquare}(y \mid \bm{x}; \bm{\theta}_1^*, \bm{\theta}_2^*, \bm{\theta}_c^*)&  \nonumber \\  
  &~~~~+ \int_{\mathcal{X}} dP_X(\bm{x}) \int_0^{\infty} dy \int_0^y dz g_{\scriptscriptstyle\square}(z, y \mid \bm{x}; \bm{\theta}_1^*, \bm{\theta}_2^*, \bm{\theta}_c^*) \log g_{\scriptscriptstyle\square}(z, y \mid \bm{x}; \bm{\theta}_1^*, \bm{\theta}_2^*, \bm{\theta}_c^*)&  \nonumber \\  
  &\leq \int_{\mathcal{X}} dP_X(\bm{x}) \int_0^{\infty} dy g_{\scriptscriptstyle\blacksquare}(y \mid \bm{x}; \bm{\theta}_1^*, \bm{\theta}_2^*, \bm{\theta}_c^*) \log g_{\scriptscriptstyle\blacksquare}(y \mid \bm{x}; \bm{\theta}_1^*, \bm{\theta}_2^*, \bm{\theta}_c^*)&  \nonumber \\  
  &~~~~+ \int_{\mathcal{X}} dP_X(\bm{x}) \int_0^{\infty} dy \int_0^y dz g_{\scriptscriptstyle\square}(z, y \mid \bm{x}; \bm{\theta}_1^*, \bm{\theta}_2^*, \bm{\theta}_c^*) \log g_{\scriptscriptstyle\square}(z, y \mid \bm{x}; \bm{\theta}_1^*, \bm{\theta}_2^*, \bm{\theta}_c^*).&  
\end{align} 
Consequently, applying the equality condition in Jensen's inequality, the estimator of $(\bm{\theta}_1, \bm{\theta}_2, \bm{\theta}_c)$ obtained by maximizing $\mathcal{E}(\bm{\theta}_1, \bm{\theta}_2) + \mathcal{L}_c(\bm{\theta}_c)$ is asymptotically consistent.

\subsection{Proof of \Cref{mix20}} 
For $y = \tau$, $v = 0$, and $\bm{x} \in \mathcal{X}$, the conditional joint distribution of $(\bar{Y}, \bar{V})$ is given by   
\begin{align} 
  &\Pr(\bar{Y} = \tau, \bar{V} = 0 \mid X = \bm{x})   \nonumber \\ 
  &= \Pr(\bar{Y} = \tau, \bar{W} = 0, \bar{V} = 0 \mid X = \bm{x}) + \Pr(\bar{Y} = \tau, \bar{W} = 1, \bar{V} = 0 \mid X = \bm{x}) \nonumber  \\ 
  &= \Pr(\tau \leq C, \tau < T_1 \mid X = \bm{x}) + \Pr(\tau \leq C, T_1 \leq \tau, \tau < T_1 + T_2 \mid X = \bm{x}) \nonumber  \\ 
  &= \int_{\tau}^{\infty} dc \int_{\tau}^{\infty} dt_1 f_1(t_1 \mid \bm{x}) f_c(c \mid \bm{x}) 
   + \int_{\tau}^{\infty} dc \int_0^{\tau} dt_1 \int_{\tau - t_1}^{\infty} dt_2 f_1(t_1 \mid \bm{x}) f_2(t_2 \mid \bm{x}) f_c(c \mid \bm{x}) \nonumber  \\  
  &= S_1(\tau \mid \bm{x}) S_c(\tau \mid \bm{x}) + \int_0^{\tau} f_1(t_1 \mid \bm{x}) S_2(\tau - t_1 \mid \bm{x}) S_c(\tau \mid \bm{x}) dt_1 \nonumber  \\   
  &= S_{\scriptscriptstyle\blacksquare}(\tau \mid \bm{x}) S_c(\tau \mid \bm{x}).  
\end{align} 
For $0 \leq y < \tau$, $v = 0$, and $\bm{x} \in \mathcal{X}$, the conditional joint distribution of $(\bar{Y}, \bar{V})$ is given by   
\begin{align} 
  \frac{\partial}{\partial y} \Pr(\bar{Y} \leq y, \bar{V} = 0 \mid X = \bm{x}) 
  &= \frac{\partial}{\partial y} \Pr(Y \leq y, V = 0 \mid X = \bm{x}) \nonumber \\ 
  &= S_{\scriptscriptstyle\blacksquare}(y \mid \bm{x}) f_c(y \mid \bm{x}).   
\end{align} 
For $0 \leq z \leq y \leq \tau$, $w = v = 1$, and $\bm{x} \in \mathcal{X}$, the conditional joint distribution of $(\bar{Z}, \bar{Y}, \bar{W}, \bar{V})$ is given by 
\begin{align} 
  \frac{\partial^2}{\partial y \partial z} \Pr(\bar{Z} \leq z, \bar{Y} \leq y, \bar{W} = 1, \bar{V} = 1 \mid X = \bm{x}) 
  &= \frac{\partial^2}{\partial y \partial z} \Pr(Z \leq z, Y \leq y, V = 1 \mid X = \bm{x}) \nonumber \\ 
  &= f_{\scriptscriptstyle\square}(z, y \mid \bm{x}) S_c(y \mid \bm{x}).   
\end{align}

\subsection{Proof of \Cref{loglikelihood_with}} 
Applying \Cref{mix20}, the log likelihood with administrative censoring is given by 
\begin{align} 
  &\log \prod_{i = 1}^n \bar{g}_{\scriptscriptstyle\blacksquare}(y_i \mid \bm{x}_i)^{1 - v_i} g_{\scriptscriptstyle\square}(z_i, y_i \mid \bm{x}_i)^{v_i}&  \nonumber \\     
  &= \log \prod_{i = 1}^n \bigl(S_{\scriptscriptstyle\blacksquare}(y_i \mid \bm{x}_i) S_c(y_i \mid \bm{x}_i)\bigr)^{(1 - v_i) \mathbb{I}(y_i = \tau)} \bigl(S_{\scriptscriptstyle\blacksquare}(y_i \mid \bm{x}_i) f_c(y_i \mid \bm{x}_i)\bigr)^{(1 - v_i) \mathbb{I}(y_i < \tau)} \bigl(f_{\scriptscriptstyle\square}(z_i, y_i \mid \bm{x}_i) S_c(y_i \mid \bm{x}_i)\bigr)^{v_i}&  \nonumber \\      
  &= \log \prod_{i = 1}^n S_{\scriptscriptstyle\blacksquare}(y_i \mid \bm{x}_i)^{1 - v_i} f_{\scriptscriptstyle\square}(z_i, y_i \mid \bm{x}_i)^{v_i} f_c(y_i \mid \bm{x}_i)^{1 - v_i^\dagger} S_c(y_i \mid \bm{x}_i)^{v_i^\dagger}&  \nonumber \\         
  &= \mathcal{E}(\bm{\theta}_1, \bm{\theta}_2) + \mathcal{L}_c^\dagger(\bm{\theta}_c).& 
\end{align}

\subsection{Proof of \Cref{ind20}} 
Under \Cref{known0}, 
if $f_{\scriptscriptstyle\square}(z, y \mid \bm{x}; \bm{\theta}_1, \bm{\theta}_2) = f_{\scriptscriptstyle\square}(z, y \mid \bm{x}; \bm{\theta}_1^\prime, \bm{\theta}_2^\prime)$ for Lebesgue-almost every $0 \leq z \leq y \leq \tau$ and $P_X$-almost every $\bm{x} \in \mathcal{X}$, 
then $(\bm{\theta}_1, \bm{\theta}_2) = (\bm{\theta}_1^\prime, \bm{\theta}_2^\prime)$. 
The proof proceeds as in \Cref{ind2}, with the time interval $[0,\infty)$ replaced by $[0,\tau]$ and $(g_{\scriptscriptstyle\blacksquare}, g_{\scriptscriptstyle\square})$ replaced by $(\bar{g}_{\scriptscriptstyle\blacksquare}, \bar{g}_{\scriptscriptstyle\square})$, accounting for the probability mass at $y = \tau$. 
Then the estimator of $(\bm{\theta}_1, \bm{\theta}_2)$ obtained by maximizing $\mathcal{E}(\bm{\theta}_1, \bm{\theta}_2)$ is asymptotically consistent.

\subsection{Proof of \Cref{EMbelow}} 
The following holds: 
\begin{align} 
&\log \mathcal{I} \bigl[\exp\bigl(\mathcal{L}(\bm{\theta}_1 \mid \mathcal{D}_1) + \mathcal{L}(\bm{\theta}_2 \mid \mathcal{D}_2)\bigr)\bigr]& \nonumber \\  
&= \log \mathcal{I} \Bigl[\prod_{i = 1}^n f_1(z_i \mid \bm{x}_i)^{w_i} S_1(z_i \mid \bm{x}_i)^{1 - w_i} f_2(y_i - z_i \mid \bm{x}_i)^{v_i} S_2(y_i - z_i \mid \bm{x}_i)^{1 - v_i}\Bigr]& \nonumber \\ 
&= \log \mathcal{I} \Bigl[\prod_{i = 1}^m f_1(z_i \mid \bm{x}_i)^{w_i} S_1(z_i \mid \bm{x}_i)^{1 - w_i} S_2(y_i - z_i \mid \bm{x}_i)\Bigr] \prod_{i = m + 1}^n f_1(z_i \mid \bm{x}_i) f_2(y_i - z_i \mid \bm{x}_i)& \nonumber \\ 
&= \log \prod_{i = 1}^m \Bigl(S_1(y_i \mid \bm{x}_i) + \int_{0}^{y_i} f_1(z_i \mid \bm{x}_i) S_2(y_i - z_i \mid \bm{x}_i) dz_i\Bigr) \prod_{i = m + 1}^n f_1(z_i \mid \bm{x}_i) f_2(y_i - z_i \mid \bm{x}_i)& \nonumber \\ 
&= \log \prod_{i = 1}^m S_{\scriptscriptstyle\blacksquare}(y_i \mid \bm{x}_i) \prod_{i = m + 1}^n f_{\scriptscriptstyle\square}(z_i, y_i \mid \bm{x}_i)& \nonumber \\ 
&= \mathcal{E}(\bm{\theta}_1, \bm{\theta}_2).& 
\end{align} 
Therefore, applying Jensen's inequality, we have 
\begin{align} 
\mathcal{E}(\bm{\theta}_1, \bm{\theta}_2)      
&= \log \mathcal{I} \Bigl[\frac{\exp\bigl(\mathcal{L}(\bm{\theta}_1 \mid \mathcal{D}_1) + \mathcal{L}(\bm{\theta}_2 \mid \mathcal{D}_2)\bigr)}{q} q\Bigr]&  \nonumber \\    
&\geq \mathcal{I} \bigl[\bigl(\mathcal{L}(\bm{\theta}_1 \mid \mathcal{D}_1) + \mathcal{L}(\bm{\theta}_2 \mid \mathcal{D}_2) - \log q\bigr) q\bigr]&  \nonumber \\  
&= \mathcal{I} \Bigl[q \log \frac{\prod_{i = 1}^m f_1(z_i \mid \bm{x}_i)^{w_i} S_1(z_i \mid \bm{x}_i)^{1 - w_i} S_2(y_i - z_i \mid \bm{x}_i)}{q}\Bigr] + \log \prod_{i = m + 1}^n f_{\scriptscriptstyle\square}(z_i, y_i \mid \bm{x}_i).& 
\end{align} 
Additionally, for each $1 \leq i \leq m$, the following holds: 
\begin{align} 
&\mathcal{I}_i\bigl[f_1(z_i \mid \bm{x}_i)^{w_i} S_1(z_i \mid \bm{x}_i)^{1 - w_i} S_2(y_i - z_i \mid \bm{x}_i)\bigr] \nonumber \\      
&= \bigl[f_1(z_i \mid \bm{x}_i)^{w_i} S_1(z_i \mid \bm{x}_i)^{1 - w_i} S_2(y_i - z_i \mid \bm{x}_i)\bigr] \big|_{z_i = y_i, w_i = 0} \nonumber \\   
&~~~~+ \int_{0}^{y_i} \bigl[f_1(z_i \mid \bm{x}_i)^{w_i} S_1(z_i \mid \bm{x}_i)^{1 - w_i} S_2(y_i - z_i \mid \bm{x}_i)\bigr] \big|_{w_i = 1} dz_i \nonumber \\   
&= S_1(y_i \mid \bm{x}_i) + \int_0^{y_i} f_1(z_i \mid \bm{x}_i) S_2(y_i - z_i \mid \bm{x}_i) dz_i \nonumber \\  
&= S_{\scriptscriptstyle\blacksquare}(y_i \mid \bm{x}_i).  
\end{align} 
Consequently, according to the equality condition in Jensen's inequality, the lower bound defined in \Cref{LB} is maximized when \Cref{posterior} holds for each $1 \leq i \leq m$.

\subsection{Proof of \Cref{two_stage_consistency}} 
By \Cref{ind2}, under \Cref{f_identifiability}, we have
\begin{align}
(\hat{\bm{\theta}}_{1,1}, \hat{\bm{\theta}}_{2,1}) \xrightarrow{p} (\bm{\theta}_{1,1}^*, \bm{\theta}_{2,1}^*)  \qquad \text{as } n_s \to \infty. 
\end{align}
Under the assumed regularity conditions, this implies
\begin{align}
\sup_{\bm{\theta}_{1,2}, \bm{\theta}_{2,2}} \bigl|n_t^{-1} \mathcal{E}_t((\hat{\bm{\theta}}_{1,1}, \bm{\theta}_{1,2}), (\hat{\bm{\theta}}_{2,1}, \bm{\theta}_{2,2})) - n_t^{-1} \mathcal{E}_t((\bm{\theta}_{1,1}^*, \bm{\theta}_{1,2}), (\bm{\theta}_{2,1}^*, \bm{\theta}_{2,2}))\bigr| \xrightarrow{p} 0 \qquad \text{as } n_s, n_t \to \infty. 
\end{align}
Moreover, the uniform law of large numbers yields
\begin{align}
\sup_{\bm{\theta}_{1,2}, \bm{\theta}_{2,2}} \left|n_t^{-1} \mathcal{E}_t((\bm{\theta}_{1,1}^*, \bm{\theta}_{1,2}), (\bm{\theta}_{2,1}^*, \bm{\theta}_{2,2})) - M_t(\bm{\theta}_{1,2}, \bm{\theta}_{2,2})\right| \xrightarrow{p} 0 \qquad \text{as } n_t \to \infty,
\end{align}
where $M_t(\bm{\theta}_{1,2}, \bm{\theta}_{2,2})$ denotes the population target-domain objective with $(\bm{\theta}_{1,1}, \bm{\theta}_{2,1}) = (\bm{\theta}_{1,1}^*, \bm{\theta}_{2,1}^*)$. 
As in the proof of \Cref{ind20}, Jensen's inequality shows that $M_t$ is maximized at $(\bm{\theta}_{1,2}^*, \bm{\theta}_{2,2}^*)$. 
By the equality condition in Jensen's inequality, under \Cref{transfer_condition}, this maximizer is unique.
Consequently, the uniform convergence of the second-stage objective and the argmax consistency theorem imply 
\begin{align}
(\hat{\bm{\theta}}_{1,2}, \hat{\bm{\theta}}_{2,2}) \xrightarrow{p} (\bm{\theta}_{1,2}^*, \bm{\theta}_{2,2}^*) \qquad \text{as } n_s, n_t \to \infty. 
\end{align}
This establishes the consistency of the two-stage estimator.

\subsection{Proof of \Cref{variance_proposition}} 
By the standard asymptotic expansion of the first-stage estimator based on the source-domain objective, we have 
\begin{align}
\operatorname{Cov}\bigl((\hat{\bm{\theta}}_{1,1}, \hat{\bm{\theta}}_{2,1}) - (\bm{\theta}_{1,1}^*, \bm{\theta}_{2,1}^*)\bigr) = n_s^{-1} I_3^{-1} + o(n_s^{-1}). 
\end{align}
Next, the second-stage estimator satisfies 
\begin{align} 
0 
&= n_t^{-1} \nabla_{(\bm{\theta}_{1,2}, \bm{\theta}_{2,2})} \mathcal{E}_t\bigl((\hat{\bm{\theta}}_{1,1}, \hat{\bm{\theta}}_{2,1}), (\bm{\theta}_{1,2}, \bm{\theta}_{2,2})\bigr) |_{(\bm{\theta}_{1,2}, \bm{\theta}_{2,2}) = (\hat{\bm{\theta}}_{1,2}, \hat{\bm{\theta}}_{2,2})}& \nonumber \\ 
&= n_t^{-1} \nabla_{(\bm{\theta}_{1,2}, \bm{\theta}_{2,2})} \mathcal{E}_t\bigl((\bm{\theta}_{1,1}^*, \bm{\theta}_{2,1}^*), (\bm{\theta}_{1,2}, \bm{\theta}_{2,2})\bigr) |_{(\bm{\theta}_{1,2}, \bm{\theta}_{2,2}) = (\bm{\theta}_{1,2}^*, \bm{\theta}_{2,2}^*)}& \nonumber \\ 
&~~~~~- I_1 \bigl((\hat{\bm{\theta}}_{1,2}, \hat{\bm{\theta}}_{2,2}) - (\bm{\theta}_{1,2}^*, \bm{\theta}_{2,2}^*)\bigr) - I_2 \bigl((\hat{\bm{\theta}}_{1,1}, \hat{\bm{\theta}}_{2,1}) - (\bm{\theta}_{1,1}^*, \bm{\theta}_{2,1}^*)\bigr) + o_p(n_t^{-1/2} + n_s^{-1/2}).& 
\end{align} 
Therefore, we have 
\begin{align}
\bigl((\hat{\bm{\theta}}_{1,2}, \hat{\bm{\theta}}_{2,2}) - (\bm{\theta}_{1,2}^*, \bm{\theta}_{2,2}^*)\bigr) 
&= n_t^{-1} I_1^{-1} \nabla_{(\bm{\theta}_{1,2}, \bm{\theta}_{2,2})} \mathcal{E}_t\bigl((\bm{\theta}_{1,1}^*, \bm{\theta}_{2,1}^*), (\bm{\theta}_{1,2}, \bm{\theta}_{2,2})\bigr) |_{(\bm{\theta}_{1,2}, \bm{\theta}_{2,2}) = (\bm{\theta}_{1,2}^*, \bm{\theta}_{2,2}^*)}& \nonumber \\  
&~~~~~- I_1^{-1} I_2 \bigl((\hat{\bm{\theta}}_{1,1}, \hat{\bm{\theta}}_{2,1}) - (\bm{\theta}_{1,1}^*, \bm{\theta}_{2,1}^*)\bigr) + o_p(n_t^{-1/2} + n_s^{-1/2}).&  
\end{align}
Consequently, we obtain 
\begin{align}
\operatorname{Cov}\bigl((\hat{\bm{\theta}}_{1,2}, \hat{\bm{\theta}}_{2,2}) - (\bm{\theta}_{1,2}^*, \bm{\theta}_{2,2}^*)\bigr)  
&= n_t^{-2} I_1^{-1} \operatorname{Cov}(\nabla_{(\bm{\theta}_{1,2}, \bm{\theta}_{2,2})} \mathcal{E}_t\bigl((\bm{\theta}_{1,1}^*, \bm{\theta}_{2,1}^*), (\bm{\theta}_{1,2}, \bm{\theta}_{2,2})\bigr) |_{(\bm{\theta}_{1,2}, \bm{\theta}_{2,2}) = (\bm{\theta}_{1,2}^*, \bm{\theta}_{2,2}^*)}) I_1^{-1}  \nonumber  \\  
&~~~~~+ I_1^{-1} I_2 \operatorname{Cov}\bigl((\hat{\bm{\theta}}_{1,1}, \hat{\bm{\theta}}_{2,1}) - (\bm{\theta}_{1,1}^*, \bm{\theta}_{2,1}^*)\bigr) I_2^\top I_1^{-1} + o(n_t^{-1} + n_s^{-1})  \nonumber  \\ 
&= n_t^{-1} I_1^{-1} + n_s^{-1} I_1^{-1} I_2 I_3^{-1} I_2^\top I_1^{-1} + o(n_t^{-1} + n_s^{-1}).  
\end{align}

\subsection{Proof of \Cref{MLE}} 
Let $\eta$ be defined as 
\begin{align} 
\eta(\bm{\gamma}) 
&\coloneqq \sum_{i = 1}^m \frac{\partial \phi(\bm{x}_i; \bm{\gamma})}{\partial \bm{\gamma}} \int_0^{y_i} h_b(t \mid \bm{x}_i) S_2(y_i - t \mid \bm{x}_i) dt& \nonumber \\  
&~~~~- \sum_{i = 1}^m \frac{\frac{\partial \phi(\bm{x}_i; \bm{\gamma})}{\partial \bm{\gamma}} \int_0^{y_i} \bigl(1 + \phi(\bm{x}_i; \bm{\gamma}) \int_t^{y_i} h_b(s \mid \bm{x}_i) ds\bigr) h_b(t \mid \bm{x}_i) \exp \bigl(\int_t^{y_i} h_b(s \mid \bm{x}_i) \phi(\bm{x}_i; \bm{\gamma}) ds\bigr) S_2(y_i - t \mid \bm{x}_i) dt}{1 + \int_0^{y_i} h_b(t \mid \bm{x}_i) \phi(\bm{x}_i; \bm{\gamma}) \exp \bigl(\int_t^{y_i} h_b(s \mid \bm{x}_i) \phi(\bm{x}_i; \bm{\gamma}) ds\bigr) S_2(y_i - t \mid \bm{x}_i) dt}.&   
\end{align}  
Then the gradient of $\mathcal{E}_t((\hat{\bm{\alpha}}, \bm{\gamma}), \hat{\bm{\theta}}_2)$ can be represented as 
\begin{align} 
\label{Egrad}
&\frac{\partial}{\partial \bm{\gamma}} \mathcal{E}_t((\hat{\bm{\alpha}}, \bm{\gamma}), \hat{\bm{\theta}}_2)& \nonumber \\   
&= \frac{\partial}{\partial \bm{\gamma}}\Bigl(\sum_{i = 1}^m \log S_{\scriptscriptstyle\blacksquare}(y_i \mid \bm{x}_i) + \sum_{i = m + 1}^n \log \bigl(f_1(z_i \mid \bm{x}_i) f_2(y_i - z_i \mid \bm{x}_i)\bigr)\Bigr)& \nonumber \\ 
&= \frac{\partial}{\partial \bm{\gamma}}\Bigl(\sum_{i = 1}^m \log \Bigl(S_1(y_i \mid \bm{x}_i) + \int_0^{y_i} f_1(t \mid \bm{x}_i) S_2(y_i - t \mid \bm{x}_i) dt\Bigr) + \sum_{i = m + 1}^n \log f_1(z_i \mid \bm{x}_i)\Bigr)& \nonumber \\ 
&= \frac{\partial}{\partial \bm{\gamma}}\Bigl(\sum_{i = 1}^m \Bigl(\log S_1(y_i \mid \bm{x}_i) + \log \Bigl(1 + \int_0^{y_i} \frac{f_1(t \mid \bm{x}_i) S_2(y_i - t \mid \bm{x}_i)}{S_1(y_i \mid \bm{x}_i)} dt\Bigr)\Bigr) + \sum_{i = m + 1}^n \log f_1(z_i \mid \bm{x}_i)\Bigr)& \nonumber \\ 
&= \frac{\partial}{\partial \bm{\gamma}} \sum_{i = 1}^m \Bigl(- \int_0^{y_i} h_b(t \mid \bm{x}_i) \phi(\bm{x}_i; \bm{\gamma}) dt + \log \Bigl(1 + \int_0^{y_i} h_b(t \mid \bm{x}_i) \phi(\bm{x}_i; \bm{\gamma}) \exp \Bigl(\int_t^{y_i} h_b(s \mid \bm{x}_i) \phi(\bm{x}_i; \bm{\gamma}) ds\Bigr) S_2(y_i - t \mid \bm{x}_i) dt\Bigr)\Bigr)& \nonumber \\ 
&~~~~ + \frac{\partial}{\partial \bm{\gamma}} \sum_{i = m + 1}^n \Bigl(\log \bigl(h_b(z_i \mid \bm{x}_i) \phi(\bm{x}_i; \bm{\gamma})\bigr) - \int_0^{z_i} h_b(t \mid \bm{x}_i) \phi(\bm{x}_i; \bm{\gamma}) dt\Bigr)& \nonumber \\   
&= \sum_{i = 1}^n \Bigl(v_i \frac{\partial \log \phi(\bm{x}_i; \bm{\gamma})}{\partial \bm{\gamma}} - \int_0^{\tilde{y}_i} h_b(t \mid \bm{x}_i) \frac{\partial \phi(\bm{x}_i; \bm{\gamma})}{\partial \bm{\gamma}} dt\Bigr)& \nonumber \\    
&~~~~ + \sum_{i = 1}^m \frac{\frac{\partial \phi(\bm{x}_i; \bm{\gamma})}{\partial \bm{\gamma}} \int_0^{y_i} \bigl(1 + \phi(\bm{x}_i; \bm{\gamma}) \int_t^{y_i} h_b(s \mid \bm{x}_i) ds\bigr) h_b(t \mid \bm{x}_i) \exp \bigl(\int_t^{y_i} h_b(s \mid \bm{x}_i) \phi(\bm{x}_i; \bm{\gamma}) ds\bigr) S_2(y_i - t \mid \bm{x}_i) dt}{1 + \int_0^{y_i} h_b(t \mid \bm{x}_i) \phi(\bm{x}_i; \bm{\gamma}) \exp \bigl(\int_t^{y_i} h_b(s \mid \bm{x}_i) \phi(\bm{x}_i; \bm{\gamma}) ds\bigr) S_2(y_i - t \mid \bm{x}_i) dt}& \nonumber \\    
&= \bm{\zeta}(\bm{\gamma}) + \sum_{i = 1}^m \frac{\frac{\partial \phi(\bm{x}_i; \bm{\gamma})}{\partial \bm{\gamma}} \int_0^{y_i} \bigl(1 + \phi(\bm{x}_i; \bm{\gamma}) \int_t^{y_i} h_b(s \mid \bm{x}_i) ds\bigr) h_b(t \mid \bm{x}_i) \exp \bigl(\int_t^{y_i} h_b(s \mid \bm{x}_i) \phi(\bm{x}_i; \bm{\gamma}) ds\bigr) S_2(y_i - t \mid \bm{x}_i) dt}{1 + \int_0^{y_i} h_b(t \mid \bm{x}_i) \phi(\bm{x}_i; \bm{\gamma}) \exp \bigl(\int_t^{y_i} h_b(s \mid \bm{x}_i) \phi(\bm{x}_i; \bm{\gamma}) ds\bigr) S_2(y_i - t \mid \bm{x}_i) dt}& \nonumber \\    
&= \tilde{\bm{\zeta}}(\bm{\gamma}) - \eta(\bm{\gamma}).& 
\end{align}  
We analyze $\eta(\bm{\gamma})$ as follows: 
\begin{align} 
|\eta(\bm{\gamma})|   
&\leq \sum_{i = 1}^m \biggl|\frac{\partial \phi(\bm{x}_i; \bm{\gamma})}{\partial \bm{\gamma}} \Bigl(\int_0^{y_i} h_b(t \mid \bm{x}_i) S_2(y_i - t \mid \bm{x}_i) dt\Bigr) \nonumber \\   
&~~~~~~~~~~~~\times \Bigl(1 + \int_0^{y_i} h_b(t \mid \bm{x}_i) \phi(\bm{x}_i; \bm{\gamma}) \exp \Bigl(\int_t^{y_i} h_b(s \mid \bm{x}_i) \phi(\bm{x}_i; \bm{\gamma}) ds\Bigr) S_2(y_i - t \mid \bm{x}_i) dt\Bigr) \nonumber \\           
&~~~~~~- \int_0^{y_i} \frac{\partial \phi(\bm{x}_i; \bm{\gamma})}{\partial \bm{\gamma}} \Bigl(1 + \phi(\bm{x}_i; \bm{\gamma}) \int_t^{y_i} h_b(s \mid \bm{x}_i) ds\Bigr) h_b(t \mid \bm{x}_i) \exp \Bigl(\int_t^{y_i} h_b(s \mid \bm{x}_i) \phi(\bm{x}_i; \bm{\gamma}) ds\Bigr) S_2(y_i - t \mid \bm{x}_i) dt\biggr| \nonumber \\  
&\leq \sum_{i = 1}^m \Bigl|\frac{\partial \phi(\bm{x}_i; \bm{\gamma})}{\partial \bm{\gamma}}\Bigr| \Bigl\{\int_0^{y_i} \Bigl(\exp \Bigl(\int_t^{y_i} h_b(s \mid \bm{x}_i) \phi(\bm{x}_i; \bm{\gamma}) ds\Bigr) - 1\Bigr) h_b(t \mid \bm{x}_i) S_2(y_i - t \mid \bm{x}_i) dt \nonumber \\           
&~~+ \int_0^{y_i} \Bigl(\int_t^{y_i} h_b(s \mid \bm{x}_i) ds + \int_0^{y_i} h_b(s \mid \bm{x}_i) S_2(y_i - s \mid \bm{x}_i) ds\Bigr) \nonumber \\ 
&~~~~~~~~~~~~\times h_b(t \mid \bm{x}_i) \phi(\bm{x}_i; \bm{\gamma}) \exp \Bigl(\int_t^{y_i} h_b(s \mid \bm{x}_i) \phi(\bm{x}_i; \bm{\gamma}) ds\Bigr) S_2(y_i - t \mid \bm{x}_i) dt\Bigr\} \nonumber \\ 
&\leq \sum_{i = 1}^m \Bigl|\frac{\partial \phi(\bm{x}_i; \bm{\gamma})}{\partial \bm{\gamma}}\Bigr| \Bigl\{\int_0^{y_i} \Bigl(\exp \Bigl(\frac{\nu_i (y_i - t)}{\phi(\bm{x}_i; \hat{\bm{\gamma}})} \phi(\bm{x}_i; \bm{\gamma})\Bigr) - 1\Bigr) \frac{\nu_i}{\phi(\bm{x}_i; \hat{\bm{\gamma}})} S_2(y_i - t \mid \bm{x}_i) dt \nonumber \\           
&~~+ \int_0^{y_i} \Bigl(\frac{\nu_i (y_i - t)}{\phi(\bm{x}_i; \hat{\bm{\gamma}})} + \int_0^{y_i} \frac{\nu_i}{\phi(\bm{x}_i; \hat{\bm{\gamma}})} S_2(y_i - s \mid \bm{x}_i) ds\Bigr) \frac{\nu_i}{\phi(\bm{x}_i; \hat{\bm{\gamma}})} \phi(\bm{x}_i; \bm{\gamma}) \exp \Bigl(\frac{\nu_i (y_i - t)}{\phi(\bm{x}_i; \hat{\bm{\gamma}})} \phi(\bm{x}_i; \bm{\gamma})\Bigr) S_2(y_i - t \mid \bm{x}_i) dt\Bigr\} \nonumber \\ 
&\leq \sum_{i = 1}^m \Bigl|\frac{\partial \phi(\bm{x}_i; \bm{\gamma})}{\partial \bm{\gamma}}\Bigr| \Bigl\{\int_0^{y_i - \delta_i} \Bigl(\exp \Bigl(\frac{\nu_i (y_i - t)}{\phi(\bm{x}_i; \hat{\bm{\gamma}})} \phi(\bm{x}_i; \bm{\gamma})\Bigr) - 1\Bigr) \frac{\nu_i}{\phi(\bm{x}_i; \hat{\bm{\gamma}})} S_2(y_i - t \mid \bm{x}_i) dt \nonumber \\ 
&~~+ \int_{y_i - \delta_i}^{y_i} \Bigl(\exp \Bigl(\frac{\nu_i (y_i - t)}{\phi(\bm{x}_i; \hat{\bm{\gamma}})} \phi(\bm{x}_i; \bm{\gamma})\Bigr) - 1\Bigr) \frac{\nu_i}{\phi(\bm{x}_i; \hat{\bm{\gamma}})} S_2(y_i - t \mid \bm{x}_i) dt \nonumber \\           
&~~+ \int_0^{y_i - \delta_i} \Bigl(\frac{\nu_i (y_i - t)}{\phi(\bm{x}_i; \hat{\bm{\gamma}})} + \int_0^{y_i} \frac{\nu_i}{\phi(\bm{x}_i; \hat{\bm{\gamma}})} S_2(y_i - s \mid \bm{x}_i) ds\Bigr) \frac{\nu_i}{\phi(\bm{x}_i; \hat{\bm{\gamma}})} \phi(\bm{x}_i; \bm{\gamma}) \exp \Bigl(\frac{\nu_i (y_i - t)}{\phi(\bm{x}_i; \hat{\bm{\gamma}})} \phi(\bm{x}_i; \bm{\gamma})\Bigr) S_2(y_i - t \mid \bm{x}_i) dt \nonumber \\ 
&~~+ \int_{y_i - \delta_i}^{y_i} \Bigl(\frac{\nu_i (y_i - t)}{\phi(\bm{x}_i; \hat{\bm{\gamma}})} + \int_0^{y_i} \frac{\nu_i}{\phi(\bm{x}_i; \hat{\bm{\gamma}})} S_2(y_i - s \mid \bm{x}_i) ds\Bigr) \frac{\nu_i}{\phi(\bm{x}_i; \hat{\bm{\gamma}})} \phi(\bm{x}_i; \bm{\gamma}) \exp \Bigl(\frac{\nu_i (y_i - t)}{\phi(\bm{x}_i; \hat{\bm{\gamma}})} \phi(\bm{x}_i; \bm{\gamma})\Bigr) S_2(y_i - t \mid \bm{x}_i) dt\Bigr\}.  
\end{align} 
By the definition of $\delta_i$, we have  
\begin{align} 
\delta_i = \int_\delta^{y_i} \frac{S_2(t \mid \bm{x}_i; \hat{\bm{\theta}}_2)}{\exp (- \nu_i t)} dt = \int_0^{y_i - \delta} \frac{S_2(y_i - t \mid \bm{x}_i; \hat{\bm{\theta}}_2)}{\exp (- (y_i - t) \nu_i)} dt \geq \int_0^{y_i - \delta} S_2(y_i - t \mid \bm{x}_i; \hat{\bm{\theta}}_2) dt. 
\end{align} 
By substituting $\bm{\gamma}$ with $\hat{\bm{\gamma}}$, we have 
\begin{align} 
|\eta(\hat{\bm{\gamma}})|
&\leq \sum_{i = 1}^m \bm{\pi}_i \Bigl\{\bigl(\exp (\nu_i y_i) - 1\bigr) \nu_i \delta_i + \bigl(\exp (\nu_i \delta_i) - 1\bigr) \nu_i \delta_i \nonumber \\           
&~~~~~~~~~~~~~~~~+ \Bigl(y_i + \int_0^{y_i} S_2(y_i - s \mid \bm{x}_i) ds\Bigr) \nu_i^2 \delta_i + \Bigl(\delta_i + \int_0^{y_i} S_2(y_i - s \mid \bm{x}_i) ds\Bigr) \nu_i^2 \delta_i \exp \bigl(\nu_i \delta_i\bigr)\Bigr\} \nonumber \\ 
&= \sum_{i = 1}^m \bm{\pi}_i \Bigl\{\bigl(\exp (\nu_i y_i) - 1\bigr) \nu_i \delta_i + \bigl(\exp (\nu_i \delta_i) - 1\bigr) \nu_i \delta_i + \nu_i^2 y_i \delta_i + \nu_i^2 \delta_i^2 \exp \bigl(\nu_i \delta_i\bigr) \nonumber \\           
&~~~~~~~~~~~~~~~~+ \Bigl(\int_0^{y_i - \delta_i} S_2(y_i - s \mid \bm{x}_i) ds + \int_{y_i - \delta_i}^{y_i} S_2(y_i - s \mid \bm{x}_i) ds\Bigr) \bigl(\nu_i^2 \delta_i + \nu_i^2 \delta_i \exp (\nu_i \delta_i)\bigr)\Bigr\} \nonumber \\ 
&\leq \sum_{i = 1}^m \bm{\pi}_i \Bigl\{\bigl(\exp (\nu_i y_i) - 1\bigr) \nu_i \delta_i + \bigl(\exp (\nu_i \delta_i) - 1\bigr) \nu_i \delta_i + \nu_i^2 y_i \delta_i + \nu_i^2 \delta_i^2 \exp \bigl(\nu_i \delta_i\bigr) 
        + 2 \bigl(1 + \exp \bigl(\nu_i \delta_i\bigr)\bigr) \nu_i^2 \delta_i^2\Bigr\}. 
\end{align} 
Consequently, under \Cref{ynu}, we obtain 
\begin{align} 
\eta(\hat{\bm{\gamma}})               
&= \mathcal{O}\Bigl(\sum_{i = 1}^m \Bigl(\bigl(\exp (\nu_i y_i) - 1\bigr) \nu_i \delta_i + \bigl(\exp (\nu_i \delta_i) - 1\bigr) \nu_i \delta_i + \nu_i^2 y_i \delta_i + \nu_i^2 \delta_i^2 \exp \bigl(\nu_i \delta_i\bigr)        
                + 2 \bigl(1 + \exp \bigl(\nu_i \delta_i\bigr)\bigr) \nu_i^2 \delta_i^2\Bigr) \bm{\pi}_i\Bigr)& \nonumber \\ 
&= \mathcal{O}\Bigl(\sum_{i = 1}^m \Bigl(\nu_i^2 y_i \delta_i + \nu_i^2 \delta_i^2 + \nu_i^2 y_i \delta_i + \nu_i^2 \delta_i^2 \bigl(1 + \nu_i \delta_i\bigr) + 2 \bigl(2 + \nu_i \delta_i\bigr) \nu_i^2 \delta_i^2\Bigr) \bm{\pi}_i\Bigr)& \nonumber \\  
&= \mathcal{O}\Bigl(\sum_{i = 1}^m \bigl(\nu_i^2 y_i \delta_i + \nu_i^2 \delta_i^2 + \nu_i^3 \delta_i^3\bigr) \bm{\pi}_i\Bigr)& \nonumber \\  
&= \mathcal{O}\Bigl(\sum_{i = 1}^m \nu_i^2 \delta_i y_i \bm{\pi}_i\Bigr).& 
\end{align}  
Consequently, we have 
\begin{align} 
\tilde{\bm{\zeta}}(\hat{\bm{\gamma}}) = \mathcal{O}\Bigl(\sum_{i = 1}^m \nu_i^2 \delta_i y_i \bm{\pi}_i\Bigr). 
\end{align}

\section{Alternative Problem Setting} 
\label{APS}  
Let $\check{Z} \coloneqq \min\{T_1, C_1\}, \check{Y} \coloneqq \min\{T_1 + T_2, C_1 + C_2\}, \check{W} \coloneqq \mathbb{I}(T_1 \leq C_1)$, and $\check{V} \coloneqq \mathbb{I}(T_1 + T_2 \leq C_1 + C_2)$, 
where $C_1$ and $C_2$ are non-negative random variables. 
Assume that $T_1$, $T_2$, and $(C_1, C_2)$ are mutually independent conditional on $X$. 
Define $f_{c_1,c_2}$ as the joint density function of $(C_1, C_2)$. 
\begin{proposition} 
\label{alternative} 
For $0 \leq y$, $v = 0$, and $\bm{x} \in \mathcal{X}$, the conditional joint distribution of $(\check{Y}, \check{V})$ is given by
\begin{align} 
  S_{\scriptscriptstyle\blacksquare}(y \mid \bm{x}) \int_0^y f_{c_1,c_2}(c_1, y - c_1 \mid \bm{x}) dc_1. 
\end{align} 
For $0 \leq z \leq y$, $(w, v) = (0, 1)$, and $\bm{x} \in \mathcal{X}$, the conditional joint distribution of $(\check{Z}, \check{Y}, \check{W}, \check{V})$ is given by
\begin{align} 
  \int_z^y f_{\scriptscriptstyle\square}(t, y \mid \bm{x}) dt \int_{y - z}^{\infty} f_{c_1,c_2}(z, c_2 \mid \bm{x}) dc_2. 
\end{align} 
For $0 \leq z \leq y$, $(w, v) = (1, 1)$, and $\bm{x} \in \mathcal{X}$, the conditional joint distribution of $(\check{Z}, \check{Y}, \check{W}, \check{V})$ is given by 
\begin{align} 
  f_{\scriptscriptstyle\square}(z, y \mid \bm{x}) \int_z^{\infty} dc_1 \int_{\max\{y - c_1, 0\}}^{\infty} dc_2 f_{c_1,c_2}(c_1, c_2 \mid \bm{x}).  
\end{align} 
\end{proposition}
\begin{proof} 
For $0 \leq y$, $v = 0$, and $\bm{x} \in \mathcal{X}$, the following holds: 
\begin{align} 
  &\frac{\partial}{\partial y} \Pr(\check{Y} \leq y, \check{V} = v \mid X = \bm{x})  \nonumber  \\ 
  &= \frac{\partial}{\partial y} \Pr(C_1 + C_2 \leq y, C_1 + C_2 < T_1 + T_2 \mid X = \bm{x}) \nonumber  \\ 
  &= \frac{\partial}{\partial y} \int_0^{\infty} dc_1 \int_0^{\max \{y - c_1, 0\}} dc_2 \int_0^{\infty} dt_1 \int_{\max \{c_1 + c_2 - t_1, 0\}}^{\infty} dt_2 f_1(t_1 \mid \bm{x}) f_2(t_2 \mid \bm{x}) f_{c_1,c_2}(c_1, c_2 \mid \bm{x}) \nonumber  \\  
  &= \frac{\partial}{\partial y} \int_0^y dc_1 \int_0^{y - c_1} dc_2 \int_0^{c_1 + c_2} dt_1 \int_{c_1 + c_2 - t_1}^{\infty} dt_2 f_1(t_1 \mid \bm{x}) f_2(t_2 \mid \bm{x}) f_{c_1,c_2}(c_1, c_2 \mid \bm{x}) \nonumber  \\  
  &~~~~ + \frac{\partial}{\partial y} \int_0^y dc_1 \int_0^{y - c_1} dc_2 \int_{c_1 + c_2}^{\infty} dt_1 \int_0^{\infty} dt_2 f_1(t_1 \mid \bm{x}) f_2(t_2 \mid \bm{x}) f_{c_1,c_2}(c_1, c_2 \mid \bm{x}) \nonumber  \\  
  &= \int_0^y dc_1 \int_0^y dt_1 \int_{y - t_1}^{\infty} dt_2 f_1(t_1 \mid \bm{x}) f_2(t_2 \mid \bm{x}) f_{c_1,c_2}(c_1, y - c_1 \mid \bm{x}) \nonumber  \\  
  &~~~~ + \int_0^y dc_1 \int_y^{\infty} dt_1 \int_0^{\infty} dt_2 f_1(t_1 \mid \bm{x}) f_2(t_2 \mid \bm{x}) f_{c_1,c_2}(c_1, y - c_1 \mid \bm{x}) \nonumber  \\  
  &= S_{\scriptscriptstyle\blacksquare}(y \mid \bm{x}) \int_0^y f_{c_1,c_2}(c_1, y - c_1 \mid \bm{x}) dc_1.  
\end{align} 
For $0 \leq z \leq y$, $(w, v) = (0, 1)$, and $\bm{x} \in \mathcal{X}$, the following holds: 
\begin{align} 
  &\frac{\partial^2}{\partial y \partial z} \Pr(\check{Z} \leq z, \check{Y} \leq y, \check{W} = w, \check{V} = v \mid X = \bm{x})& \nonumber  \\  
  &= \frac{\partial^2}{\partial y \partial z} \Pr(C_1 \leq z, T_1 + T_2 \leq y, C_1 < T_1, T_1 + T_2 \leq C_1 + C_2 \mid X = \bm{x})& \nonumber  \\ 
  &= \frac{\partial^2}{\partial y \partial z} \int_0^z dc_1 \int_{c_1}^{\infty} dt_1 \int_0^{\max \{y - t_1, 0\}} dt_2 \int_{t_1 + t_2 - c_1}^{\infty} dc_2 f_1(t_1 \mid \bm{x}) f_2(t_2 \mid \bm{x}) f_{c_1,c_2}(c_1, c_2 \mid \bm{x})& \nonumber  \\  
  &= \frac{\partial^2}{\partial y \partial z} \int_0^z dc_1 \int_{c_1}^y dt_1 \int_0^{y - t_1} dt_2 \int_{t_1 + t_2 - c_1}^{\infty} dc_2 f_1(t_1 \mid \bm{x}) f_2(t_2 \mid \bm{x}) f_{c_1,c_2}(c_1, c_2 \mid \bm{x})& \nonumber  \\  
  &= \frac{\partial}{\partial y} \int_z^y dt_1 \int_0^{y - t_1} dt_2 \int_{t_1 + t_2 - z}^{\infty} dc_2 f_1(t_1 \mid \bm{x}) f_2(t_2 \mid \bm{x}) f_{c_1,c_2}(z, c_2 \mid \bm{x})& \nonumber  \\  
  &= \int_z^y f_{\scriptscriptstyle\square}(t, y \mid \bm{x}) dt \int_{y - z}^{\infty} f_{c_1,c_2}(z, c_2 \mid \bm{x}) dc_2.& 
\end{align} 
For $0 \leq z \leq y$, $(w, v) = (1, 1)$, and $\bm{x} \in \mathcal{X}$, the following holds: 
\begin{align} 
  &\frac{\partial^2}{\partial y \partial z} \Pr(\check{Z} \leq z, \check{Y} \leq y, \check{W} = w, \check{V} = v \mid X = \bm{x}) \nonumber  \\  
  &= \frac{\partial^2}{\partial y \partial z} \Pr(T_1 \leq z, T_1 + T_2 \leq y, T_1 \leq C_1, T_1 + T_2 \leq C_1 + C_2 \mid X = \bm{x}) \nonumber  \\ 
  &= \frac{\partial^2}{\partial y \partial z} \int_0^z dt_1 \int_0^{y - t_1} dt_2 \int_{t_1}^{\infty} dc_1 \int_{\max \{t_1 + t_2 - c_1, 0\}}^{\infty} dc_2 f_1(t_1 \mid \bm{x}) f_2(t_2 \mid \bm{x}) f_{c_1,c_2}(c_1, c_2 \mid \bm{x}) \nonumber  \\  
  &= \frac{\partial^2}{\partial y \partial z} \int_0^z dt_1 \int_0^{y - t_1} dt_2 \int_{t_1}^{t_1 + t_2} dc_1 \int_{t_1 + t_2 - c_1}^{\infty} dc_2 f_1(t_1 \mid \bm{x}) f_2(t_2 \mid \bm{x}) f_{c_1,c_2}(c_1, c_2 \mid \bm{x}) \nonumber  \\  
  &~~~~ + \frac{\partial^2}{\partial y \partial z} \int_0^z dt_1 \int_0^{y - t_1} dt_2 \int_{t_1 + t_2}^{\infty} dc_1 \int_0^{\infty} dc_2 f_1(t_1 \mid \bm{x}) f_2(t_2 \mid \bm{x}) f_{c_1,c_2}(c_1, c_2 \mid \bm{x}) \nonumber  \\  
  &= \frac{\partial}{\partial y} \int_0^{y - z} dt_2 \int_z^{z + t_2} dc_1 \int_{z + t_2 - c_1}^{\infty} dc_2 f_1(z \mid \bm{x}) f_2(t_2 \mid \bm{x}) f_{c_1,c_2}(c_1, c_2 \mid \bm{x}) \nonumber  \\  
  &~~~~ + \frac{\partial}{\partial y} \int_0^{y - z} dt_2 \int_{z + t_2}^{\infty} dc_1 \int_0^{\infty} dc_2 f_1(z \mid \bm{x}) f_2(t_2 \mid \bm{x}) f_{c_1,c_2}(c_1, c_2 \mid \bm{x}) \nonumber  \\  
  &= \int_z^y dc_1 \int_{y - c_1}^{\infty} dc_2 f_1(z \mid \bm{x}) f_2(y - z \mid \bm{x}) f_{c_1,c_2}(c_1, c_2 \mid \bm{x}) + \int_y^{\infty} dc_1 \int_0^{\infty} dc_2 f_1(z \mid \bm{x}) f_2(y - z \mid \bm{x}) f_{c_1,c_2}(c_1, c_2 \mid \bm{x}) \nonumber  \\  
  &= f_{\scriptscriptstyle\square}(z, y \mid \bm{x}) \int_z^{\infty} dc_1 \int_{\max\{y - c_1, 0\}}^{\infty} dc_2 f_{c_1,c_2}(c_1, c_2 \mid \bm{x}). 
\end{align} 
\end{proof} 
\begin{proposition} 
Let $\check{\mathcal{E}}$ be defined as 
\begin{align} 
\check{\mathcal{E}}(\bm{\theta}_1, \bm{\theta}_2) \coloneqq \sum_{i=1}^m \log S_{\scriptscriptstyle\blacksquare}(y_i \mid \bm{x}_i) + \sum_{i = m + 1}^n \Bigl(w_i \log f_{\scriptscriptstyle\square}(z_i, y_i \mid \bm{x}_i) + (1 - w_i) \log \int_{z_i}^{y_i} f_{\scriptscriptstyle\square}(t, y_i \mid \bm{x}_i) dt\Bigr). 
\end{align} 
Under \Cref{f_identifiability}, the estimator of $(\bm{\theta}_1, \bm{\theta}_2)$ obtained by maximizing $\check{\mathcal{E}}(\bm{\theta}_1, \bm{\theta}_2)$ in the absence of administrative censoring is asymptotically consistent. 
\end{proposition} 
\begin{proof} 
Consider the case $(W, V) = (1, 1)$ in the proof of \Cref{ind2}. 
The proposition follows immediately from the same argument. 
\end{proof}  
In this setting, the occurrence of an earlier event after the first censoring time is reported, but its timing is not observed. 
When $\Pr(C_2 = 0) = 1$, this setting is equivalent to that described in \Cref{Setting}.

\section{Supplementary Theoretical Results} 

\subsection{Restricted Class of Density Functions}
\label{example_hazard} 
Let $f(t; C) = C$ if $t \leq \tau_0$; otherwise, let $f(t; C) = (1 - C \tau_0) f_0(t)$,  
where $\tau_0$ is a known parameter satisfying $\tau \leq \tau_0$, $C$ is an unknown parameter satisfying $0 < C < \tau_0^{-1}$, and $f_0(t)$ is a known positive function such that $\int_{\tau_0}^\infty f_0(t) dt = 1$. 
A single function from this class is identifiable on $[0,\tau]$, since $f(t;C) = C$ on this interval. 
In contrast, the product of two functions from this class is not identifiable on $[0,\tau]$: only the product $C_1 C_2$ is observed, so different pairs satisfying $C_1 C_2 = C_1^\prime C_2^\prime$ yield the same product. 
For example, whenever the admissibility constraints are satisfied, $(C_1, C_2) = (1, 1)$ and $(C_1^\prime, C_2^\prime) = (0.5, 2)$ yield the same product on $[0, \tau]$. 
Therefore, this case does not satisfy \Cref{known0}.

\subsection{Bias Induced by Ignoring Reporting Delays} 
\label{bias} 
The second term in $S_{\scriptscriptstyle\blacksquare}(y \mid \bm{x})$ captures the probability that the event occurrence is unreported. 
The following proposition shows that omitting this term can introduce bias in the estimation of $\bm{\gamma}$. 
\begin{assumption} 
\label{componentwise} 
For each $1 \leq i \leq m$, the gradient $\frac{\partial}{\partial \bm{\gamma}} \phi(\bm{x}_i; \bm{\gamma})|_{\bm{\gamma} = \hat{\bm{\gamma}}}$ is strictly positive componentwise. 
\end{assumption} 
\begin{proposition} 
Under \Cref{componentwise}, $\bm{\zeta}(\hat{\bm{\gamma}})$ is strictly negative componentwise for $m \geq 1$. 
\end{proposition} 
\begin{proof} 
As in \Cref{Egrad}, the gradient of $\mathcal{E}_t((\hat{\bm{\alpha}}, \bm{\gamma}), \hat{\bm{\theta}}_2)$ can be represented as 
\begin{align} 
\bm{\zeta}(\bm{\gamma}) + \sum_{i = 1}^m \frac{\frac{\partial \phi(\bm{x}_i; \bm{\gamma})}{\partial \bm{\gamma}} \int_0^{y_i} \bigl(1 + \phi(\bm{x}_i; \bm{\gamma}) \int_t^{y_i} h_b(s \mid \bm{x}_i) ds\bigr) h_b(t \mid \bm{x}_i) \exp \bigl(\int_t^{y_i} h_b(s \mid \bm{x}_i) \phi(\bm{x}_i; \bm{\gamma}) ds\bigr) S_2(y_i - t \mid \bm{x}_i) dt}{1 + \int_0^{y_i} h_b(t \mid \bm{x}_i) \phi(\bm{x}_i; \bm{\gamma}) \exp \bigl(\int_t^{y_i} h_b(s \mid \bm{x}_i) \phi(\bm{x}_i; \bm{\gamma}) ds\bigr) S_2(y_i - t \mid \bm{x}_i) dt}.   
\end{align}  
Under \Cref{componentwise}, the last factor on the right-hand side is strictly positive componentwise.  
Consequently, $\bm{\zeta}(\hat{\bm{\gamma}})$ is strictly negative componentwise. 
\end{proof} 
\Cref{componentwise} holds, for example, when $\phi(\bm{x}; \bm{\gamma}) = \exp(\bm{\gamma}^\top \bm{x})$ and $\bm{x}$ are strictly positive componentwise.

\section{Supplementary Experiments} 

\subsection{Computational Environment} 
\label{CE} 
We used a 64-bit Windows machine with the Intel Core i9-9900K @ 3.60 GHz processor and 64 GB of RAM. 
The experiments were implemented in Python 3.7.3 (Python Software Foundation License Version 2), using lifelines 0.27.8 (MIT License) and SciPy 1.2.1 (BSD 3-Clause License). 

\subsection{Data Preprocessing} 
\label{DP} 
For the variable \textit{rx} in the \texttt{Colon} dataset, we encoded the levels \textit{Obs}, \textit{Lev}, and \textit{Lev+5FU} as 0, 1, and 2, respectively. 
For the variable \textit{size} in the \texttt{Rotterdam} dataset, we encoded the levels $<=20$, $20-50$, and $>50$ as 0, 1, and 2, respectively. 
All variables were standardized. 
In the second stage of the transfer-learning procedure, we included an intercept term in the covariate vector. 

\subsection{Parameter Estimation Performance} 
\label{AR} 
The coverage probabilities for the source-domain parameters range from $0.914$ to $0.958$. 
Those for the target-domain parameters range from $0.894$ to $0.970$, except for $\gamma_1$ or the Toy1 setting with $\tau = 0.25$.
Thus, the estimated standard errors generally yield coverage probabilities reasonably close to the nominal $95\%$ level. 

\begin{table*}[h]
\caption{
Parameter estimation results for \textit{Ours} in the toy model. 
For each parameter, we report the bias, root mean squared error (RMSE), mean estimated standard error (MeanSE), and empirical coverage probability of the nominal $95\%$ Wald confidence interval. 
The asymptotic covariance in \Cref{covariance_matrix} is modified to account for the approximate estimating equation. 
We conducted $500$ trials. 
} 
\label{confidence_interval}
\begin{center}
\begin{tabular}{cccccccccc} 
\toprule
\multirow{2.5}{*}{$\tau$} & \multirow{2.5}{*}{Parameter} & \multicolumn{4}{c}{Toy1} & \multicolumn{4}{c}{Toy2} \\ 
\cmidrule(lr){3-6} \cmidrule(lr){7-10} & & Bias & RMSE & MeanSE & Coverage & Bias & RMSE & MeanSE & Coverage \\ 
\midrule 
\multirow{14}{*}{$\infty$} & $\log \lambda$ & 0.014 & 0.178 & 0.166 & 0.924 & 0.001 & 0.100 & 0.094 & 0.942 \\ 
 & $\log \alpha_{1}$ & -0.111 & 0.272 & 0.239 & 0.936 & -0.050 & 0.182 & 0.169 & 0.944 \\
 & $\log \alpha_{2}$ & -0.092 & 0.259 & 0.238 & 0.938 & -0.034 & 0.172 & 0.164 & 0.942 \\
 & $\log \alpha_{3}$ & -0.029 & 0.259 & 0.238 & 0.926 & -0.013 & 0.143 & 0.144 & 0.958 \\
 & $\beta_{1}$ & 0.051 & 0.236 & 0.218 & 0.928 & 0.005 & 0.160 & 0.151 & 0.932 \\
 & $\beta_{2}$ & 0.008 & 0.236 & 0.208 & 0.914 & -0.006 & 0.154 & 0.146 & 0.930 \\
 & $\beta_{3}$ & 0.007 & 0.218 & 0.208 & 0.934 & 0.002 & 0.148 & 0.145 & 0.948 \\
 & $\beta_{4}$ & -0.001 & 0.222 & 0.208 & 0.928 & -0.006 & 0.157 & 0.145 & 0.920 \\
 & $\beta_{5}$ & 0.009 & 0.226 & 0.208 & 0.936 & -0.008 & 0.152 & 0.145 & 0.946 \\
 & $\beta_{6}$ & -0.017 & 0.211 & 0.208 & 0.954 & 0.009 & 0.155 & 0.145 & 0.932 \\
 & $\beta_{7}$ & -0.005 & 0.220 & 0.208 & 0.936 & -0.001 & 0.155 & 0.145 & 0.946 \\
 & $\beta_{8}$ & 0.009 & 0.240 & 0.209 & 0.922 & 0.003 & 0.156 & 0.145 & 0.942 \\
 & $\beta_{9}$ & -0.005 & 0.213 & 0.208 & 0.948 & -0.001 & 0.157 & 0.145 & 0.934 \\
 & $\beta_{10}$ & -0.007 & 0.219 & 0.208 & 0.940 & 0.012 & 0.158 & 0.145 & 0.934 \\
\midrule 
\multirow{10}{*}{$0.25$} & $\gamma_{1}$ & -0.522 & 0.799 & 0.501 & 0.734 & -0.194 & 0.439 & 0.363 & 0.872 \\
 & $\gamma_{2}$ & -0.009 & 0.590 & 0.480 & 0.868 & 0.007 & 0.344 & 0.343 & 0.952 \\
 & $\gamma_{3}$ & -0.028 & 0.609 & 0.480 & 0.828 & 0.002 & 0.343 & 0.343 & 0.936 \\
 & $\gamma_{4}$ & 0.016 & 0.575 & 0.475 & 0.868 & -0.011 & 0.344 & 0.345 & 0.950 \\
 & $\gamma_{5}$ & 0.024 & 0.613 & 0.475 & 0.824 & 0.011 & 0.359 & 0.340 & 0.936 \\
 & $\gamma_{6}$ & -0.025 & 0.610 & 0.475 & 0.834 & -0.021 & 0.375 & 0.341 & 0.900 \\
 & $\gamma_{7}$ & 0.004 & 0.591 & 0.480 & 0.842 & 0.009 & 0.352 & 0.344 & 0.940 \\
 & $\gamma_{8}$ & -0.023 & 0.592 & 0.482 & 0.844 & 0.006 & 0.347 & 0.342 & 0.942 \\
 & $\gamma_{9}$ & 0.005 & 0.562 & 0.477 & 0.864 & -0.003 & 0.355 & 0.341 & 0.936 \\
 & $\gamma_{10}$ & 0.018 & 0.611 & 0.482 & 0.848 & 0.005 & 0.356 & 0.344 & 0.932 \\
\midrule 
\multirow{10}{*}{$1$} & $\gamma_{1}$ & -0.307 & 0.421 & 0.280 & 0.784 & -0.045 & 0.204 & 0.199 & 0.946 \\
 & $\gamma_{2}$ & -0.023 & 0.266 & 0.244 & 0.918 & 0.000 & 0.170 & 0.168 & 0.958 \\
 & $\gamma_{3}$ & 0.012 & 0.282 & 0.241 & 0.894 & 0.002 & 0.168 & 0.168 & 0.950 \\
 & $\gamma_{4}$ & -0.007 & 0.261 & 0.243 & 0.940 & 0.008 & 0.172 & 0.169 & 0.938 \\
 & $\gamma_{5}$ & 0.005 & 0.270 & 0.243 & 0.922 & 0.006 & 0.177 & 0.168 & 0.940 \\
 & $\gamma_{6}$ & -0.016 & 0.242 & 0.242 & 0.940 & -0.002 & 0.173 & 0.168 & 0.942 \\
 & $\gamma_{7}$ & 0.002 & 0.274 & 0.243 & 0.920 & 0.005 & 0.172 & 0.168 & 0.946 \\
 & $\gamma_{8}$ & 0.004 & 0.265 & 0.244 & 0.932 & -0.004 & 0.163 & 0.168 & 0.954 \\
 & $\gamma_{9}$ & -0.007 & 0.265 & 0.242 & 0.908 & -0.003 & 0.165 & 0.167 & 0.970 \\
 & $\gamma_{10}$ & 0.005 & 0.269 & 0.242 & 0.916 & 0.002 & 0.168 & 0.168 & 0.958 \\
\bottomrule 
\end{tabular}
\end{center} 
\end{table*}

\section{Broader Impacts}
\label{broader} 
This work provides a methodological contribution to survival analysis under reporting delays. 
It may support risk evaluation in applications such as biostatistics and insurance, where timely risk evaluation is important but event information may be delayed. 
Careful validation of the modeling assumptions is important, since violations of these assumptions may lead to inappropriate decision making.

\newpage
\section*{NeurIPS Paper Checklist}

\begin{enumerate}

\item {\bf Claims}
    \item[] Question: Do the main claims made in the abstract and introduction accurately reflect the paper's contributions and scope?
    \item[] Answer: \answerYes{} 
    \item[] Justification: The abstract and introduction accurately state the paper's scope and contributions, including the proposed transfer-learning procedure for survival modeling with right-censored reporting delays under administrative censoring. 
    \item[] Guidelines:
    \begin{itemize}
        \item The answer \answerNA{} means that the abstract and introduction do not include the claims made in the paper.
        \item The abstract and/or introduction should clearly state the claims made, including the contributions made in the paper and important assumptions and limitations. A \answerNo{} or \answerNA{} answer to this question will not be perceived well by the reviewers. 
        \item The claims made should match theoretical and experimental results, and reflect how much the results can be expected to generalize to other settings. 
        \item It is fine to include aspirational goals as motivation as long as it is clear that these goals are not attained by the paper. 
    \end{itemize}

\item {\bf Limitations}
    \item[] Question: Does the paper discuss the limitations of the work performed by the authors?
    \item[] Answer: \answerYes{} 
    \item[] Justification: The limitations are discussed in \Cref{Conclusion}, particularly the dependence of the proposed transfer-learning procedure on the validity of the shared-parameter assumptions. 
    \item[] Guidelines:
    \begin{itemize}
        \item The answer \answerNA{} means that the paper has no limitation while the answer \answerNo{} means that the paper has limitations, but those are not discussed in the paper. 
        \item The authors are encouraged to create a separate ``Limitations'' section in their paper.
        \item The paper should point out any strong assumptions and how robust the results are to violations of these assumptions (e.g., independence assumptions, noiseless settings, model well-specification, asymptotic approximations only holding locally). The authors should reflect on how these assumptions might be violated in practice and what the implications would be.
        \item The authors should reflect on the scope of the claims made, e.g., if the approach was only tested on a few datasets or with a few runs. In general, empirical results often depend on implicit assumptions, which should be articulated.
        \item The authors should reflect on the factors that influence the performance of the approach. For example, a facial recognition algorithm may perform poorly when image resolution is low or images are taken in low lighting. Or a speech-to-text system might not be used reliably to provide closed captions for online lectures because it fails to handle technical jargon.
        \item The authors should discuss the computational efficiency of the proposed algorithms and how they scale with dataset size.
        \item If applicable, the authors should discuss possible limitations of their approach to address problems of privacy and fairness.
        \item While the authors might fear that complete honesty about limitations might be used by reviewers as grounds for rejection, a worse outcome might be that reviewers discover limitations that aren't acknowledged in the paper. The authors should use their best judgment and recognize that individual actions in favor of transparency play an important role in developing norms that preserve the integrity of the community. Reviewers will be specifically instructed to not penalize honesty concerning limitations.
    \end{itemize}

\item {\bf Theory assumptions and proofs}
    \item[] Question: For each theoretical result, does the paper provide the full set of assumptions and a complete (and correct) proof?
    \item[] Answer: \answerYes{} 
    \item[] Justification: Each theoretical result states its assumptions explicitly, and complete proofs are provided in \Cref{proofs}. 
    \item[] Guidelines:
    \begin{itemize}
        \item The answer \answerNA{} means that the paper does not include theoretical results. 
        \item All the theorems, formulas, and proofs in the paper should be numbered and cross-referenced.
        \item All assumptions should be clearly stated or referenced in the statement of any theorems.
        \item The proofs can either appear in the main paper or the supplemental material, but if they appear in the supplemental material, the authors are encouraged to provide a short proof sketch to provide intuition. 
        \item Inversely, any informal proof provided in the core of the paper should be complemented by formal proofs provided in appendix or supplemental material.
        \item Theorems and Lemmas that the proof relies upon should be properly referenced. 
    \end{itemize}

    \item {\bf Experimental result reproducibility}
    \item[] Question: Does the paper fully disclose all the information needed to reproduce the main experimental results of the paper to the extent that it affects the main claims and/or conclusions of the paper (regardless of whether the code and data are provided or not)?
    \item[] Answer: \answerYes{} 
    \item[] Justification: \Cref{Experiment} describes the data, evaluation metric, compared methods, and implementation details needed to reproduce the main experimental results. 
    \item[] Guidelines:
    \begin{itemize}
        \item The answer \answerNA{} means that the paper does not include experiments.
        \item If the paper includes experiments, a \answerNo{} answer to this question will not be perceived well by the reviewers: Making the paper reproducible is important, regardless of whether the code and data are provided or not.
        \item If the contribution is a dataset and\slash or model, the authors should describe the steps taken to make their results reproducible or verifiable. 
        \item Depending on the contribution, reproducibility can be accomplished in various ways. For example, if the contribution is a novel architecture, describing the architecture fully might suffice, or if the contribution is a specific model and empirical evaluation, it may be necessary to either make it possible for others to replicate the model with the same dataset, or provide access to the model. In general. releasing code and data is often one good way to accomplish this, but reproducibility can also be provided via detailed instructions for how to replicate the results, access to a hosted model (e.g., in the case of a large language model), releasing of a model checkpoint, or other means that are appropriate to the research performed.
        \item While NeurIPS does not require releasing code, the conference does require all submissions to provide some reasonable avenue for reproducibility, which may depend on the nature of the contribution. For example
        \begin{enumerate}
            \item If the contribution is primarily a new algorithm, the paper should make it clear how to reproduce that algorithm.
            \item If the contribution is primarily a new model architecture, the paper should describe the architecture clearly and fully.
            \item If the contribution is a new model (e.g., a large language model), then there should either be a way to access this model for reproducing the results or a way to reproduce the model (e.g., with an open-source dataset or instructions for how to construct the dataset).
            \item We recognize that reproducibility may be tricky in some cases, in which case authors are welcome to describe the particular way they provide for reproducibility. In the case of closed-source models, it may be that access to the model is limited in some way (e.g., to registered users), but it should be possible for other researchers to have some path to reproducing or verifying the results.
        \end{enumerate}
    \end{itemize}

\item {\bf Open access to data and code}
    \item[] Question: Does the paper provide open access to the data and code, with sufficient instructions to faithfully reproduce the main experimental results, as described in supplemental material?
    \item[] Answer: \answerYes{} 
    \item[] Justification: The supplemental material includes the code and instructions needed to reproduce the main experimental results. 
    \item[] Guidelines:
    \begin{itemize}
        \item The answer \answerNA{} means that paper does not include experiments requiring code.
        \item Please see the NeurIPS code and data submission guidelines (\url{https://neurips.cc/public/guides/CodeSubmissionPolicy}) for more details.
        \item While we encourage the release of code and data, we understand that this might not be possible, so \answerNo{} is an acceptable answer. Papers cannot be rejected simply for not including code, unless this is central to the contribution (e.g., for a new open-source benchmark).
        \item The instructions should contain the exact command and environment needed to run to reproduce the results. See the NeurIPS code and data submission guidelines (\url{https://neurips.cc/public/guides/CodeSubmissionPolicy}) for more details.
        \item The authors should provide instructions on data access and preparation, including how to access the raw data, preprocessed data, intermediate data, and generated data, etc.
        \item The authors should provide scripts to reproduce all experimental results for the new proposed method and baselines. If only a subset of experiments are reproducible, they should state which ones are omitted from the script and why.
        \item At submission time, to preserve anonymity, the authors should release anonymized versions (if applicable).
        \item Providing as much information as possible in supplemental material (appended to the paper) is recommended, but including URLs to data and code is permitted.
    \end{itemize}

\item {\bf Experimental setting/details}
    \item[] Question: Does the paper specify all the training and test details (e.g., data splits, hyperparameters, how they were chosen, type of optimizer) necessary to understand the results?
    \item[] Answer: \answerYes{} 
    \item[] Justification: \Cref{Experiment} specifies all the training and test details necessary to understand the results. 
    \item[] Guidelines:
    \begin{itemize}
        \item The answer \answerNA{} means that the paper does not include experiments.
        \item The experimental setting should be presented in the core of the paper to a level of detail that is necessary to appreciate the results and make sense of them.
        \item The full details can be provided either with the code, in appendix, or as supplemental material.
    \end{itemize}

\item {\bf Experiment statistical significance}
    \item[] Question: Does the paper report error bars suitably and correctly defined or other appropriate information about the statistical significance of the experiments?
    \item[] Answer: \answerYes{} 
    \item[] Justification: \Cref{Experiment} reports the mean and standard deviation of the evaluation metrics. 
    \item[] Guidelines:
    \begin{itemize}
        \item The answer \answerNA{} means that the paper does not include experiments.
        \item The authors should answer \answerYes{} if the results are accompanied by error bars, confidence intervals, or statistical significance tests, at least for the experiments that support the main claims of the paper.
        \item The factors of variability that the error bars are capturing should be clearly stated (for example, train/test split, initialization, random drawing of some parameter, or overall run with given experimental conditions).
        \item The method for calculating the error bars should be explained (closed form formula, call to a library function, bootstrap, etc.)
        \item The assumptions made should be given (e.g., Normally distributed errors).
        \item It should be clear whether the error bar is the standard deviation or the standard error of the mean.
        \item It is OK to report 1-sigma error bars, but one should state it. The authors should preferably report a 2-sigma error bar than state that they have a 96\% CI, if the hypothesis of Normality of errors is not verified.
        \item For asymmetric distributions, the authors should be careful not to show in tables or figures symmetric error bars that would yield results that are out of range (e.g., negative error rates).
        \item If error bars are reported in tables or plots, the authors should explain in the text how they were calculated and reference the corresponding figures or tables in the text.
    \end{itemize}

\item {\bf Experiments compute resources}
    \item[] Question: For each experiment, does the paper provide sufficient information on the computer resources (type of compute workers, memory, time of execution) needed to reproduce the experiments?
    \item[] Answer: \answerYes{} 
    \item[] Justification: \Cref{CE} reports the hardware and software environment used for the experiments, including the CPU, memory, operating system, and Python version. 
    \item[] Guidelines:
    \begin{itemize}
        \item The answer \answerNA{} means that the paper does not include experiments.
        \item The paper should indicate the type of compute workers CPU or GPU, internal cluster, or cloud provider, including relevant memory and storage.
        \item The paper should provide the amount of compute required for each of the individual experimental runs as well as estimate the total compute. 
        \item The paper should disclose whether the full research project required more compute than the experiments reported in the paper (e.g., preliminary or failed experiments that didn't make it into the paper). 
    \end{itemize}
    
\item {\bf Code of ethics}
    \item[] Question: Does the research conducted in the paper conform, in every respect, with the NeurIPS Code of Ethics \url{https://neurips.cc/public/EthicsGuidelines}?
    \item[] Answer: \answerYes{} 
    \item[] Justification: We are not aware of any ethical concerns that would conflict with the NeurIPS Code of Ethics. 
    \item[] Guidelines:
    \begin{itemize}
        \item The answer \answerNA{} means that the authors have not reviewed the NeurIPS Code of Ethics.
        \item If the authors answer \answerNo, they should explain the special circumstances that require a deviation from the Code of Ethics.
        \item The authors should make sure to preserve anonymity (e.g., if there is a special consideration due to laws or regulations in their jurisdiction).
    \end{itemize}

\item {\bf Broader impacts} 
    \item[] Question: Does the paper discuss both potential positive societal impacts and negative societal impacts of the work performed?
    \item[] Answer: \answerYes{} 
    \item[] Justification: \Cref{broader} discusses potential positive impacts of more reliable survival analysis under reporting delays, as well as potential risks arising from model misspecification or inappropriate use in real-world decision-making. 
    \item[] Guidelines:
    \begin{itemize}
        \item The answer \answerNA{} means that there is no societal impact of the work performed.
        \item If the authors answer \answerNA{} or \answerNo, they should explain why their work has no societal impact or why the paper does not address societal impact.
        \item Examples of negative societal impacts include potential malicious or unintended uses (e.g., disinformation, generating fake profiles, surveillance), fairness considerations (e.g., deployment of technologies that could make decisions that unfairly impact specific groups), privacy considerations, and security considerations.
        \item The conference expects that many papers will be foundational research and not tied to particular applications, let alone deployments. However, if there is a direct path to any negative applications, the authors should point it out. For example, it is legitimate to point out that an improvement in the quality of generative models could be used to generate Deepfakes for disinformation. On the other hand, it is not needed to point out that a generic algorithm for optimizing neural networks could enable people to train models that generate Deepfakes faster.
        \item The authors should consider possible harms that could arise when the technology is being used as intended and functioning correctly, harms that could arise when the technology is being used as intended but gives incorrect results, and harms following from (intentional or unintentional) misuse of the technology.
        \item If there are negative societal impacts, the authors could also discuss possible mitigation strategies (e.g., gated release of models, providing defenses in addition to attacks, mechanisms for monitoring misuse, mechanisms to monitor how a system learns from feedback over time, improving the efficiency and accessibility of ML).
    \end{itemize}
    
\item {\bf Safeguards}
    \item[] Question: Does the paper describe safeguards that have been put in place for responsible release of data or models that have a high risk for misuse (e.g., pre-trained language models, image generators, or scraped datasets)?
    \item[] Answer: \answerNA{} 
    \item[] Justification: The paper does not release data or models that pose a high risk of misuse. 
    \item[] Guidelines:
    \begin{itemize}
        \item The answer \answerNA{} means that the paper poses no such risks.
        \item Released models that have a high risk for misuse or dual-use should be released with necessary safeguards to allow for controlled use of the model, for example by requiring that users adhere to usage guidelines or restrictions to access the model or implementing safety filters. 
        \item Datasets that have been scraped from the Internet could pose safety risks. The authors should describe how they avoided releasing unsafe images.
        \item We recognize that providing effective safeguards is challenging, and many papers do not require this, but we encourage authors to take this into account and make a best faith effort.
    \end{itemize}

\item {\bf Licenses for existing assets}
    \item[] Question: Are the creators or original owners of assets (e.g., code, data, models), used in the paper, properly credited and are the license and terms of use explicitly mentioned and properly respected?
    \item[] Answer: \answerYes{} 
    \item[] Justification: The software licenses are specified in \Cref{CE}, and the sources of the publicly available datasets are described in \Cref{Experiment}. 
    \item[] Guidelines:
    \begin{itemize}
        \item The answer \answerNA{} means that the paper does not use existing assets.
        \item The authors should cite the original paper that produced the code package or dataset.
        \item The authors should state which version of the asset is used and, if possible, include a URL.
        \item The name of the license (e.g., CC-BY 4.0) should be included for each asset.
        \item For scraped data from a particular source (e.g., website), the copyright and terms of service of that source should be provided.
        \item If assets are released, the license, copyright information, and terms of use in the package should be provided. For popular datasets, \url{paperswithcode.com/datasets} has curated licenses for some datasets. Their licensing guide can help determine the license of a dataset.
        \item For existing datasets that are re-packaged, both the original license and the license of the derived asset (if it has changed) should be provided.
        \item If this information is not available online, the authors are encouraged to reach out to the asset's creators.
    \end{itemize}

\item {\bf New assets}
    \item[] Question: Are new assets introduced in the paper well documented and is the documentation provided alongside the assets?
    \item[] Answer: \answerNA{} 
    \item[] Justification: No new dataset, pretrained model, benchmark, or other external asset is introduced. The code used to reproduce the experiments is provided in the supplementary material. 
    \item[] Guidelines:
    \begin{itemize}
        \item The answer \answerNA{} means that the paper does not release new assets.
        \item Researchers should communicate the details of the dataset\slash code\slash model as part of their submissions via structured templates. This includes details about training, license, limitations, etc. 
        \item The paper should discuss whether and how consent was obtained from people whose asset is used.
        \item At submission time, remember to anonymize your assets (if applicable). You can either create an anonymized URL or include an anonymized zip file.
    \end{itemize}

\item {\bf Crowdsourcing and research with human subjects}
    \item[] Question: For crowdsourcing experiments and research with human subjects, does the paper include the full text of instructions given to participants and screenshots, if applicable, as well as details about compensation (if any)? 
    \item[] Answer: \answerNA{} 
    \item[] Justification: The paper does not involve crowdsourcing, participant recruitment, or direct interaction with human subjects. 
    \item[] Guidelines:
    \begin{itemize}
        \item The answer \answerNA{} means that the paper does not involve crowdsourcing nor research with human subjects.
        \item Including this information in the supplemental material is fine, but if the main contribution of the paper involves human subjects, then as much detail as possible should be included in the main paper. 
        \item According to the NeurIPS Code of Ethics, workers involved in data collection, curation, or other labor should be paid at least the minimum wage in the country of the data collector. 
    \end{itemize}

\item {\bf Institutional review board (IRB) approvals or equivalent for research with human subjects}
    \item[] Question: Does the paper describe potential risks incurred by study participants, whether such risks were disclosed to the subjects, and whether Institutional Review Board (IRB) approvals (or an equivalent approval/review based on the requirements of your country or institution) were obtained?
    \item[] Answer: \answerNA{} 
    \item[] Justification: The paper does not involve participant recruitment, intervention, or direct interaction with human subjects, and uses only existing publicly available datasets. 
    \item[] Guidelines:
    \begin{itemize}
        \item The answer \answerNA{} means that the paper does not involve crowdsourcing nor research with human subjects.
        \item Depending on the country in which research is conducted, IRB approval (or equivalent) may be required for any human subjects research. If you obtained IRB approval, you should clearly state this in the paper. 
        \item We recognize that the procedures for this may vary significantly between institutions and locations, and we expect authors to adhere to the NeurIPS Code of Ethics and the guidelines for their institution. 
        \item For initial submissions, do not include any information that would break anonymity (if applicable), such as the institution conducting the review.
    \end{itemize}

\item {\bf Declaration of LLM usage}
    \item[] Question: Does the paper describe the usage of LLMs if it is an important, original, or non-standard component of the core methods in this research? Note that if the LLM is used only for writing, editing, or formatting purposes and does \emph{not} impact the core methodology, scientific rigor, or originality of the research, declaration is not required.
    \item[] Answer: \answerNA{} 
    \item[] Justification: LLMs are not used as part of the core methodology of this research. 
    \item[] Guidelines:
    \begin{itemize}
        \item The answer \answerNA{} means that the core method development in this research does not involve LLMs as any important, original, or non-standard components.
        \item Please refer to our LLM policy in the NeurIPS handbook for what should or should not be described.
    \end{itemize}

\end{enumerate}

\end{document}